\documentclass[10pt]{article}

\usepackage[accepted]{tmlr}

\usepackage{amsmath,amsfonts,bm}

\def\eqref#1{equation~\ref{#1}}

\def\1{\bm{1}}

\DeclareMathAlphabet{\mathsfit}{\encodingdefault}{\sfdefault}{m}{sl}
\SetMathAlphabet{\mathsfit}{bold}{\encodingdefault}{\sfdefault}{bx}{n}

\newcommand{\E}{\mathbb{E}}
\newcommand{\Ls}{\mathcal{L}}

\usepackage{hyperref}
\hypersetup{pdftitle={[Re] Boosting the Visual Interpretability of CLIP via Adversarial Fine-Tuning}, pdfauthor={Anton Nuzhdin, Andrei Gruzitski, Balázs Egyed, Aleksandr Raudvee}}
\usepackage{url}
\usepackage{graphicx}
\usepackage{booktabs}
\usepackage{array}
\usepackage{multirow}
\usepackage{xcolor}
\usepackage{soul}
\usepackage{subcaption}
\usepackage{algorithm}
\usepackage{algpseudocode}
\usepackage{amsthm}
\newtheorem{proposition}{Proposition}

\title{[Re] Boosting the Visual Interpretability of CLIP\\ via Adversarial Fine-Tuning}

\author{
\name Anton Nuzhdin\footnotemark[1], Andrei Gruzitski\footnotemark[1], Balázs Egyed\footnotemark[1], Aleksandr Raudvee\footnotemark[1] \\
\addr \{anton.nuzhdin, andrei.gruzitski, balazs.egyed, alex.raudvee\}@student.uva.nl \\
\addr University of Amsterdam
}

\def\month{09}
\def\year{2026}
\def\openreview{\url{https://openreview.net/forum?id=uPXRBZfkYy}}

\definecolor{revcolor}{rgb}{0,0,0.0}
\newcommand{\rev}[1]{#1}

\begin{document}

\maketitle
\renewcommand{\thefootnote}{\fnsymbol{footnote}}
\footnotetext[1]{These authors contributed equally.}
\renewcommand{\thefootnote}{\arabic{footnote}}

\begin{abstract}
This paper presents a reproducibility study of ``Boosting the Visual Interpretability of CLIP via Adversarial Fine-Tuning'' by \citet{gong2025boosting}, published at ICLR 2025, which proposes an unsupervised adversarial fine-tuning (AFT) method with norm regularisation to enhance the visual interpretability of CLIP's image encoder. We attempt to reproduce the key claims regarding improved saliency map quality, increased concept alignment, transferability to out-of-distribution datasets, and the trade-off with zero-shot accuracy. Beyond reproduction, we propose a saliency-guided regularisation extension that introduces an Energy Pointing Game loss, directly supervising the spatial alignment of Simple Gradient saliency maps with target objects. We evaluate our extension across a range of saliency-loss weights and show that explicit saliency supervision \rev{substantially} improves localisation metrics 
\rev{at a modest cost: the reduction in zero-shot accuracy and adversarial robustness is small at the lowest saliency-loss weight ($w_s=0.2$) and moderate at the strongest setting ($w_s=0.6$)}. We restructure the codebase for configurable, systematic experimentation. Our reproduction code is available at \url{https://github.com/Andrei3223/Re-Boosting-CLIP-Interpretability}.
\end{abstract}

\section{Introduction}
CLIP (Contrastive Language-Image Pre-training) \citep{radford2021learning} has become a foundational model for vision-language representation learning and is widely used as a backbone in modern multimodal systems. Despite its strong performance, the internal representations of CLIP's image encoder remain difficult to interpret. In particular, common gradient-based attribution methods (e.g., Simple Gradient and Grad-CAM; \citealp{selvaraju2017grad}) often produce noisy and unstable saliency maps, and the encoder itself appears to contain only a limited number of human-interpretable concept detectors.

To understand why these saliency maps are unreliable, it helps to recall how they are constructed. Gradient-based attribution methods quantify the sensitivity of the model's output to small input perturbations. For a model $f_\theta$ with parameters $\theta$, the Simple Gradient saliency map \citep{simonyan2013deep} is defined as $S(x) = |\nabla_x \Ls(f_\theta(x), y)|$, where $\Ls$ is a loss function. However, these gradients in pretrained CLIP models tend to be dense and noisy, distributing attribution broadly across the input rather than focusing on semantically meaningful regions.

\citet{gong2025boosting} address the problem of noisy, diffuse gradients in pretrained CLIP by proposing adversarial fine-tuning (AFT) with norm regularisation, which enforces sparse encoding of visual features. They provide a theoretical analysis showing that AFT implicitly regularises input gradients, leading to more concise and interpretable saliency maps, and their experiments demonstrate significant improvements across multiple interpretability metrics. However, because AFT relies solely on the adversarial objective, it does not explicitly enforce alignment between saliency maps and semantically meaningful regions of the input.

In this work, we conduct a thorough reproducibility study of \citet{gong2025boosting}, verifying their central claims about saliency map quality, concept alignment, transferability, and the accuracy-robustness trade-off, and we extend their methodology with explicit spatial supervision. Our main contributions are:
\begin{itemize}
    \item \textbf{Reproduction of the main claims.} We independently verify the four central claims of the original paper regarding saliency map quality, neuron-concept alignment, transferability to out-of-distribution datasets and downstream tasks, and the trade-off between zero-shot accuracy and adversarial robustness.
    \item \textbf{Implementation of evaluation protocols.} Since the original repository does not release code for a broad range of evaluation strategies, we implement these evaluations from scratch and release them as part of our reproducible codebase.

      \item \textbf{Saliency-guided regularisation extension.} We propose a training objective that augments AFT with an Energy Pointing Game \rev{(EPG) loss, explicitly aligning} Simple Gradient saliency maps with \rev{ground-truth object regions. The proposed Saliency-Guided regularisation consistently improves localisation} while preserving most of \rev{AFT's adversarial robustness. We further show that, to leading order, the EPG objective is equivalent to minimising a background-restricted adversarial objective, providing an adversarial interpretation of explicit saliency supervision (Appendix~\ref{sec:theory-epg}).}

\end{itemize}
\section{Scope of Reproducibility}
\label{sec:claims}

All our experiments aim to reproduce the improvement in model interpretability achieved by training with the proposed AFT method. To this end, we focus on reproducing the following main claims from \citet{gong2025boosting}:

\begin{itemize}
    \item \textbf{Claim 1 (Saliency Map Quality):} AFT significantly improves the quality of Simple Gradient and Grad-CAM saliency maps, as measured by Pointing Game accuracy and related localisation metrics (Table 1, Figure 2a,b in the original paper).

    \item \textbf{Claim 2 (Concept Alignment):} AFT increases neuron-concept alignment in CLIP's visual encoder.

    \item \textbf{Claim 3 (Transferability):} The interpretability improvements transfer to out-of-distribution datasets and downstream tasks, such as linear probing (Figure 2c,d, Figure 3 in the original paper).

    \item \textbf{Claim 4 (Accuracy Trade-off):} AFT maintains reasonable zero-shot classification accuracy while improving interpretability, with only a 2.84\% drop at $\varepsilon=1/255$ (Table 2 in the original paper).
\end{itemize}

Each experiment in Section~\ref{sec:results} is designed to evaluate at least one of these claims.


Beyond these reproduction targets, we evaluate an extension: saliency-guided \rev{regularisation that} aligns Simple Gradient attribution maps with ground-truth masks during adversarial fine-tuning, \rev{testing whether} explicit spatial supervision \rev{improves saliency quality beyond AFT's implicit gradient regularisation} without sacrificing adversarial robustness.

\section{Methodology}

To reproduce the study, we \rev{build on the authors'} publicly available \href{https://github.com/peterant330/CLIP_AFT}{repository}\rev{, which provides training scripts,} PGD \citep{pgd} and APGD \citep{apgd} \rev{attacks, and saliency visualisation utilities. It} does not include code for dataset-level evaluation, \rev{adversarial evaluation, or the localisation metrics (PG}, energy-based PG, \rev{PA, AP, IoU), so we implement these from scratch, along with} the neuron-concept alignment evaluation \rev{of \mbox{
\citet{dissect}}\hskip0pt
, retrieval evaluation on Flickr30k, and ROAR \mbox{
\citep{hooker2019benchmark}}\hskip0pt
. }

\subsection{Unsupervised AFT and Optimization}
The original paper first formulates supervised AFT as a minimax objective over text--image similarity (Equation~\ref{eq:supervised_aft} in Appendix~\ref{appendix:AFT}). Since the text encoder is frozen, the authors derive an image-only upper bound, yielding the unsupervised AFT objective:
\begin{equation}
\min_\theta \E_{x \sim \mathcal{D}_{\text{train}}} \max_{\delta_x}
\left[
\frac{1}{2}
\left\|
\E_{z \sim \mathcal{N}(0,\sigma^2 I)}
[f_\theta(x+z+\delta_x)] - I_x
\right\|^2
- h(\delta_x)
\right],
\label{eq:unsupervised_aft}
\end{equation}
where $f_\theta$ is the trainable visual encoder, $x$ is a clean training image, and $I_x$ is the reference embedding produced by the frozen original CLIP image encoder. The perturbation $\delta_x$ is constrained within an $\ell_\infty$ budget $\varepsilon$, $z \sim \mathcal{N}(0,\sigma^2 I)$ is Gaussian smoothing noise, and $h(\delta_x)$ penalises large perturbations. Thus, the inner maximisation finds a worst-case perturbation that moves the perturbed embedding away from $I_x$, while the outer minimisation trains $f_\theta$ to preserve the reference embedding under this perturbation. We adopt this image-only objective because it is the primary formulation used in the original experiments.


The inner maximisation in Equation~\ref{eq:unsupervised_aft} is approximated with Projected Gradient Descent \rev{\citep{pgd}: starting from $\delta_x = 0$}, the perturbation \rev{takes} $T_{\text{adv}}$ \rev{ascent steps of size $\alpha$, each projected} back onto the $\ell_\infty$ ball of radius $\varepsilon$\rev{, and the final perturbation is} used for the outer update of $\theta$.



\subsection{Hyperparameters}
\label{subsec:params}
\textbf{Model description.}
Following Appendix A.3 of the original paper, we use ViT-B/16 \citep{vit} for all experiments, except for the evaluation with ROAR, which is conducted using the ResNet-50 \citep{resnet} architecture. For the ViT-B/16 baseline, we use OpenCLIP \citep{cherti2023reproducible}. During AFT, only the visual encoder is fine-tuned while the text encoder remains frozen. Due to resource constraints, we focus our reproduction on ViT-B/16, which the original paper uses for the majority of its experiments (ViT-L/14 is used primarily for zero-shot accuracy).

\textbf{Adversarial Fine-Tuning Setup.} 
Following Appendix A.3 of the original paper, we train the ViT-B/16 model with PGD on the ImageNet2012 dataset \citep{imagenet}. We train two separate models with different perturbation strengths, $\varepsilon \in \{1/255, 4/255\}$. The adversarial attack is performed in the $\ell_\infty$ norm with 10 PGD steps and a step size of $1/255$.

Training is carried out for 20,000 iterations with a linear warmup of 1{,}400 steps followed by a cosine learning rate decay. We use the AdamW optimiser with $\beta_1 = 0.9$ and $\beta_2 = 0.95$, a peak learning rate of $1\times10^{-5}$, and a weight decay of $1\times10^{-4}$. The training objective follows the unsupervised AFT formulation in Equation~\ref{eq:unsupervised_aft}, where the embedding-consistency term is implemented with an $\ell_2$ loss; training is performed with a batch size of 128. Following the original paper, we set the Gaussian smoothing parameter $\sigma = 1/255$ and the Huber smoothness parameter $\eta = 1/255$. All experiments are conducted on two NVIDIA A100 GPUs (40\,GB each). 
\rev{On this hardware, the main AFT$^4$ training run took 34.6 hours and consumed 28.7 kWh (an estimated 7.7 kg CO$_2$e), and the full saliency-guided fine-tuning pipeline added roughly 14.9 kWh (3.98 kg CO$_2$e). A complete cost and carbon accounting is given in Appendix~\ref{appendix:environmental_impact}.}

A complete summary of all datasets used in our experiments is provided in Appendix~\ref{appendix:datasets}.

\subsection{Saliency-Guided Regularisation}
\label{sec:methepg}



We introduce an auxiliary loss term that directly supervises the spatial alignment of gradient-based saliency maps with reference object masks. The key idea is to measure the concentration of input-gradient magnitude within annotated foreground regions, penalising the model when saliency mass falls on background features.

\textbf{Saliency construction.} Following \citet{gong2025boosting}, given a clean input image $x$, we compute the Simple Gradient saliency map as the input gradient of the \emph{feature-consistency score}
\[
s = I_{orig}^\top I_x,
\]
the cosine similarity between the trainable embedding $I_x = f_\theta(x)$ and a frozen reference embedding $I_{orig} = f_{orig}(x)$ from OpenAI CLIP (both unit-normalised). We then compute the input gradient $g = \nabla_x s$ using \texttt{create\_graph=True}, ensuring the gradient remains part of the differentiable computational graph. This enables the calculation of second-order derivatives during the backward pass of the saliency loss. 


The final saliency map \rev{$S = \sum_{c=1}^{C} |g_c|$ aggregates} absolute gradient magnitudes across the \rev{$C = 3$ input channels, giving a heatmap of} shape \((B, H, W)\).


\textbf{Energy Pointing Game loss (EPG-loss).} We first normalise \rev{each} saliency map into a spatial probability distribution \rev{$\hat{S}_i(p) = S_i(p) / \sum_{p'} S_i(p')$ over} the $H \times W$ \rev{grid}. Let $M_i \subseteq \{1,\dots,H\}\times\{1,\dots,W\}$ be the binarised foreground mask of image $i$ from COCO instance annotations. The EPG loss is then

\[
\mathcal{L}_{\text{EPG}} = 1 - \frac{1}{B} \sum_{i=1}^{B} \sum_{p \in M_i} \hat{S}_i(p).
\]
By construction, \(\mathcal{L}_{\text{EPG}} \in [0,1]\): \(\mathcal{L}_{\text{EPG}}=0\) indicates that all saliency mass is concentrated inside the target mask, whereas \(\mathcal{L}_{\text{EPG}}=1\) indicates that no saliency mass falls within it.

\textbf{Integration into training.} The full training objective is
\[
\mathcal{L}_{\text{total}} =
\mathcal{L}_{\text{adv}}
+ w_s \mathcal{L}_{\text{EPG}},
\]
where \(w_s\) controls the contribution of the saliency regularisation term. Since \(\mathcal{L}_{\text{EPG}}\) is added independently of the clean/adversarial trade-off, its strength can be tuned separately, and setting \(w_s = 0\) recovers the standard adversarial training objective. 
\rev{Appendix~\ref{sec:theory-epg} shows that, to first order in the perturbation size, minimising $\mathcal{L}_{\text{EPG}}$ enforces robustness to background-restricted adversarial perturbations, a spatially localised case of the duality behind AFT.}

\textbf{Data.}
We utilise the COCO dataset \citep{lin2015microsoftcococommonobjects}, specifically leveraging its high-quality instance-level segmentation masks to provide ground-truth supervision for the saliency maps. Both images and their corresponding binary masks are resized to $224 \times 224$ pixels to match the input resolution of the CLIP vision transformer. 

While the original AFT formulation of \citet{gong2025boosting} is entirely unsupervised, our saliency-guided extension introduces a supervised component by explicitly aligning the model's gradient-based attention with human-annotated foreground regions. 

\textbf{Training configuration.} We found that training with the saliency loss from scratch does not converge. 
Before AFT has converged, input gradients are dense and near-uniform, so the second-order signal that drives $\nabla_\theta \mathcal{L}_{\text{EPG}}$ has effectively random sign across pixels and produces high-variance updates. We therefore apply the saliency loss only as a fine-tuning stage on top of an already-converged checkpoint. Specifically, starting from an AFT$^4$ checkpoint trained for \(5{,}000\) steps on COCO, we fine-tune for \(2{,}000\) additional steps with a learning rate of \(5 \times 10^{-6}\) and a warmup of \(100\) steps. To study the sensitivity of the method to the strength of the saliency regulariser, we sweep \(w_s \in \{0, 0.2, 0.4, 0.6\}\). All configurations use a batch size of \(128\), which fits on \(40\,\mathrm{GB}\) A100 GPUs.

\rev{Unlike AFT, which needs only images, the EPG loss requires dense instance masks and a second-order backward pass through the input gradient, which raises the step time to $1.27\times$ that of plain AFT fine-tuning in our runs. The extension's main practical cost is the supervision requirement: it restricts training to mask-annotated data such as COCO, whereas AFT trains on any image collection. Weaker supervision (boxes, or pseudo-masks from a segmenter) would broaden the usable training data and is left to future work.
}
\section{Results}
\label{sec:results}

\subsection{Results Reproducing Original Paper}

\subsubsection{Saliency Map Quality (Claim 1)}
\label{sec:claim1}

\textbf{Qualitative Analysis.} We reproduced the qualitative saliency map comparison from the original paper \citep{gong2025boosting} (Figure 2a,b), using our retrained ViT-B/16 model with $\varepsilon = 4/255$. Following the original methodology, we selected images from the COCO validation set \citep{lin2015microsoftcococommonobjects} containing multiple objects and generated saliency maps using different text prompts.

\begin{figure}[h]
\centering
\includegraphics[width=0.68\textwidth]{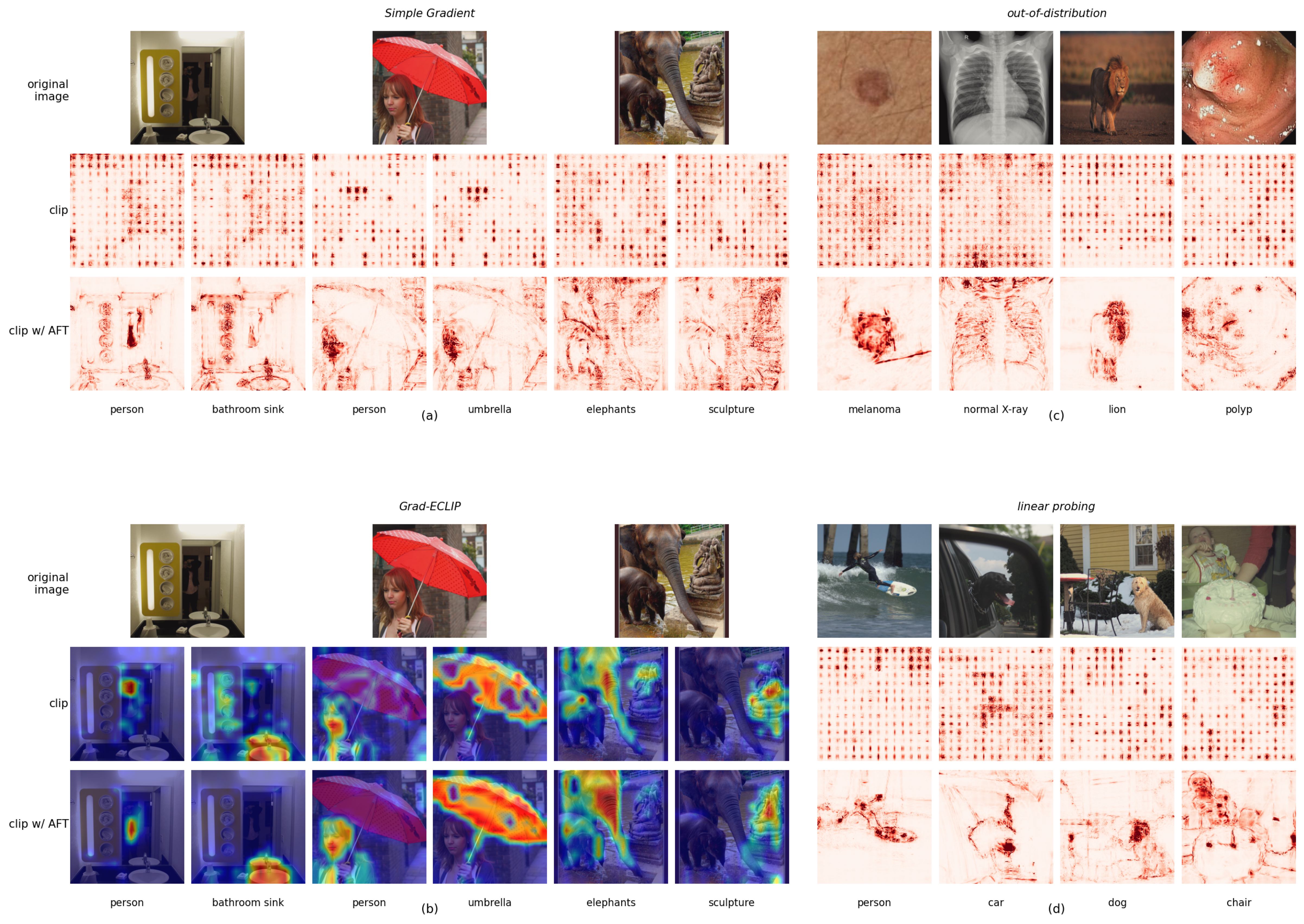}
\caption{Saliency map comparison between CLIP (middle row) and CLIP with AFT (bottom
 row). (a) Simple Gradient maps for different text prompts. (b) Grad-ECLIP visualisations for the same images. (c) Out-of-distribution evaluation on medical and wildlife images. (d) Linear probing results on COCO classification.}
\label{fig:gradient_maps}
\end{figure}

Figure~\ref{fig:gradient_maps}(a) shows Simple Gradient saliency maps for three image-prompt pairs. The original CLIP model produced dense, noisy gradients distributed broadly across the input, while CLIP with AFT generated substantially cleaner maps that concentrated on semantically relevant regions. We additionally evaluated Grad-ECLIP~\citep{zhao2024gradient}, a CLIP-tailored variant of class activation mapping, shown in Figure~\ref{fig:gradient_maps}(b). Consistent with the Simple Gradient results, AFT produced sharper, prompt-sensitive maps that localised on the referenced object, while the baseline CLIP maps were diffuse and less prompt-sensitive. These visualisations closely resemble the qualitative results reported in the original paper and show the benefit of AFT. We also ran standard Grad-CAM for completeness and found the results inconsistent: \rev{the maps collapse into a few blocky high-intensity patches whose placement varies between images. Running the identical pipeline on an ImageNet-pretrained ResNet-50 yields smooth, object-centred maps (Appendix~\ref{appendix:gradcam_sanity}), so the instability comes from applying Grad-CAM's channel-averaging heuristic to the coarse $14\times14$ token grid of a contrastively trained ViT, not from our implementation. Quantitatively, Grad-CAM localisation still improves under AFT (Table~\ref{tab:localization}); given the instability of the underlying maps, we treat the Grad-CAM half of Claim 1 as inconclusive, neither confirming nor contradicting the original result.} Overall, the qualitative analysis \rev{supports the Simple Gradient component of Claim 1, with Grad-ECLIP providing convergent evidence.}

\textbf{Quantitative Analysis.} 
We assess performance quantitatively using five localisation metrics, all computed against ground-truth segmentation masks: Pointing Game (PG), which measures whether the argmax saliency pixel falls inside the object mask; Energy Pointing Game (E-PG) \citep{wang2020score}, the fraction of total saliency mass contained within the mask; Average Precision (AP), the area under the precision-recall curve treating saliency values as soft predictions; Pixel Accuracy (PA) and Intersection-over-Union (IoU), computed on binarised saliency maps. Several protocol details are not formally documented in the original work; \rev{our choices are described below and stress-tested in Appendix~\ref{appendix:protocol_sensitivity}.}

\begin{table*}[!h]
\centering
\caption{Evaluation of localisation ability using PG \citep{zhang2018top}, E-PG \citep{wang2020score}, and segmentation metrics (PA, AP and Mask IoU) on the ImageNet-Segmentation validation dataset. \rev{Absolute values differ from those reported by \citet{gong2025boosting} owing to undocumented evaluation-protocol choices (Appendix~\ref{appendix:protocol_sensitivity}); see text.}}
\label{tab:localization}
\resizebox{0.85\textwidth}{!}{
\begin{tabular}{ll|ccccc|ccccc}
\hline
\multirow{2}{*}{Saliency} & \multirow{2}{*}{CLIP} 
& \multicolumn{5}{c|}{Zero-shot classification} 
& \multicolumn{5}{c}{Linear probing} \\
& 
& PG$\uparrow$ & E-PG$\uparrow$ & AP$\uparrow$ & PA$\uparrow$ & IoU$\uparrow$
& PG$\uparrow$ & E-PG$\uparrow$ & AP$\uparrow$ & PA$\uparrow$ & IoU$\uparrow$ \\
\hline
 & w/o AFT
& \rev{25.26} & \rev{29.21} & \rev{31.65} & \rev{61.15} & \rev{12.19}
& \rev{28.05} & \rev{30.95} & \rev{33.10} & \rev{61.96} & \rev{13.37} \\
SG & w/ AFT$^{\rev{4}}$ (ours)
& \rev{56.55} & \rev{\textbf{46.22}} & \rev{49.95} & \rev{68.71} & \rev{24.43}
& \rev{57.83} & \rev{\textbf{46.99}} & \rev{50.67} & \rev{68.97} & \rev{24.87} \\
 & w/ AFT$^{\rev{4}}$ (original)
& \rev{\textbf{56.61}} & \rev{46.09} & \rev{\textbf{49.98}} & \rev{\textbf{68.75}} & \rev{\textbf{24.47}}
& \rev{\textbf{57.87}} & \rev{46.84} & \rev{\textbf{50.68}} & \rev{\textbf{69.00}} & \rev{\textbf{24.90}} \\
\hline
 & w/o AFT
& \rev{26.56} & \rev{27.16} & \rev{31.58} & \rev{59.29} & \rev{10.53}
& \rev{67.10} & \rev{59.45} & \rev{61.56} & \rev{\textbf{69.97}} & \rev{31.46} \\
GC & w/ AFT$^{\rev{4}}$ (ours)
& \rev{57.57} & \rev{49.59} & \rev{54.22} & \rev{\textbf{67.29}} & \rev{26.35}
& \rev{\textbf{68.68}} & \rev{61.39} & \rev{\textbf{64.81}} & \rev{68.90} & \rev{\textbf{32.85}} \\
 & w/ AFT$^{\rev{4}}$ (original)
& \rev{\textbf{59.57}} & \rev{\textbf{52.09}} & \rev{\textbf{55.51}} & \rev{67.02} & \rev{\textbf{27.30}}
& \rev{68.16} & \rev{\textbf{61.41}} & \rev{64.35} & \rev{68.70} & \rev{32.58} \\
\hline
 & w/o AFT
& \rev{\textbf{67.37}} & \rev{52.00} & \rev{63.66} & \rev{72.92} & \rev{32.58}
& -- & -- & -- & -- & -- \\
Grad-ECLIP & w/ AFT$^{\rev{4}}$ (ours)
& \rev{66.23} & \rev{\textbf{54.01}} & \rev{\textbf{65.73}} & \rev{\textbf{73.22}} & \rev{\textbf{33.53}}
& -- & -- & -- & -- & -- \\
 & w/ AFT$^{\rev{4}}$ (original)
& \rev{65.87} & \rev{53.67} & \rev{65.30} & \rev{73.11} & \rev{33.28}
& -- & -- & -- & -- & -- \\
\hline
\end{tabular}
}
\end{table*}

Following the original protocol, we evaluate models on the intersection of the ImageNet validation set and ImageNet-Segmentation \citep{Gao_2023} (12,419 images). As shown in Table~\ref{tab:localization}, our retrained AFT$^4$ model closely matches the author-released checkpoint across both zero-shot classification and linear probing \rev{(within 0.3 points on every Simple Gradient metric, with a stable offset of about 2.0 PG points on Grad-CAM)}, confirming model-level reproducibility. However, our absolute values exhibit a uniform shift \rev{against} \citet{gong2025boosting} (e.g., a baseline Mask IoU of \rev{12.19} versus their 2.16); because our retrained model and the released checkpoint \rev{score nearly identically, the shift is most plausibly} a protocol artefact of evaluation choices omitted in the original paper. \rev{Appendix~\ref{appendix:protocol_sensitivity} sweeps the main candidates: binarisation threshold, pointing-game hit radius, mask-downsampling mode, and percentile clipping. These choices change absolute metric levels, but do not reorder the models within any saliency method. Under AFT}, Simple Gradient PG accuracy \rev{more than doubles (25.26 to 56.55)}, E-PG \rev{rises by more than half}, and Average Precision improves \rev{by a similar margin; Grad-CAM behaves alike. Grad-ECLIP is the exception: its pointing game sits at parity between the AFT models and the baseline (66.23 versus 67.37). We therefore consider the Simple Gradient component of Claim 1 reproduced directionally, while the absolute values remain protocol-dependent and do not numerically match \citet{gong2025boosting}.}

\subsubsection{Concept Alignment (Claim 2)}
\label{sec:concept}
To quantitatively evaluate the interpretability of the visual encoder, we reproduced a network dissection experiment following CLIP-Dissect \citep{dissect}, \rev{identifying neuron}s that selectively respond to human-understandable concepts.

\textbf{Experimental Setup.}
\rev{We define} neurons as the activation channels of the [CLS] token \rev{at the final ViT layer and probe them on the Broden~\citep{bau2017network} dataset}: for each neuron $k$, we record its activation $A_k$ across probing images and measure alignment with a concept similarity array $S_c$ generated by a pre-trained CLIP model; \rev{a neuron counts} as a detector for \rev{the concept with which it aligns most strongly.}

While the authors mentioned using the 20,000 most common English words, they did not specify any filtering methods or provide the exact list of words used. To address this and ensure reproducibility, we utilised the \texttt{wordfreq} \citep{wordfreq} library to extract the top 20,000 most frequent English words.

While the original study reports a single alignment metric, \rev{we benchmark three: Soft-WPMI} (\rev{the CLIP-Dissect default, which favours semantically specific concepts and penalises frequent terms}), \rev{cosine similarity, and cubed cosine similarity, a variant that amplifies peak activations and surfaces sparse detectors.}

\begin{table}[h]
\centering
\caption{Mean and median neuron-concept alignment scores for different AFT strengths.}
\label{tab:concept4}
\small
\begin{tabular}{lcccccc}
\toprule
& \multicolumn{2}{c}{w/o AFT} & \multicolumn{2}{c}{w/ AFT$^1$} & \multicolumn{2}{c}{w/ AFT$^4$} \\
Metric & Mean & Median & Mean & Median & Mean & Median \\
\midrule
Soft-WPMI & \rev{0.2264} & \rev{0.2023} & \rev{0.2610} & \rev{0.2307} & \rev{\textbf{0.2817}} & \rev{\textbf{0.2558}} \\
Cos & \rev{0.2852} & \rev{0.4437} & \rev{0.3046} & \rev{0.4726} & \rev{\textbf{0.3065}} & \rev{\textbf{0.4996}} \\
Cos-cubed & \rev{0.2369} & \rev{0.2205} & \rev{0.3051} & \rev{0.2853} & \rev{\textbf{0.3289}} & \rev{\textbf{0.3124}} \\
\bottomrule
\end{tabular}
\end{table}

\textbf{Results.} As shown in Table~\ref{tab:concept4}, AFT consistently improves neuron-concept alignment across all three similarity metrics, with mean and median scores rising monotonically with stronger AFT (\rev{full score distributions in Figure~\ref{fig:alignment_hist}, Appendix~\ref{appendix:nd_ablation}).}

We also apply the original Network Dissection protocol~\citep{bau2017network} to the final visual token of CLIP's image encoder (768 neurons, $14\times14$ spatial grid) \rev{using the same Broden dataset}. For each neuron, we estimate the top-0.5\% activation threshold, upsample the binarised activation maps to input resolution, and compute IoU with each concept's segmentation mask. A neuron counts as a detector when its best-matching concept exceeds $\text{IoU} \geq 0.04$.

\begin{table}[h]
\centering
\caption{Network Dissection on the Broden dataset (768 neurons, $\text{IoU} \geq 0.04$ threshold).}
\label{tab:net_dissect}
\small
\begin{tabular}{lccc}
\toprule
Model & Detectors & \% neurons \\
\midrule
CLIP w/o AFT & \rev{160/768} & \rev{20.8\%} \\
AFT$^1$ (ours) & \rev{279/768} & \rev{36.3\%} \\
AFT$^4$ (ours) & \rev{390/768} & \rev{50.8\%} \\
AFT$^4$ (original) & \rev{390/768} & \rev{50.8\%} \\
\bottomrule
\end{tabular}
\end{table}

AFT$^4$ more than doubles the number of concept detectors relative to the baseline (160 to 390 of 768), and our retrained AFT$^4$ matches the released checkpoint exactly at this operating point. The absolute counts differ \rev{from those reported in the original paper (42 to 159)}, likely due to unspecified details in their Network Dissection setup. \rev{To test how far the conclusion depends on such details, Appendix~\ref{appendix:nd_ablation} repeats the dissection with the concept vocabulary, activation quantile, IoU threshold, and probed layer exposed as configuration. The ordering AFT$^4 > $ AFT$^1 > $ CLIP holds in every ablation cell, while the absolute count is governed almost entirely by the IoU threshold: the baseline count falls from 572 detectors at IoU 0.02 to 64 at IoU 0.06, and both the original counts and ours fall inside the range this sweep spans, so the count gap is a threshold artefact, not a modelling difference.} We therefore view Claim~2 as reproduced in trend, but not quantitatively matched.

\subsubsection{Transferability (Claim 3)}

\textbf{Experimental Setup.} We evaluated the transferability of interpretability improvements to both out-of-distribution (OOD) domains and downstream classification tasks. Our experiments reproduced the settings of Figures 2(c,d) from the original paper.

For the OOD evaluation in Figure~\ref{fig:gradient_maps}(c), we tested the models on four datasets from different visual domains: ISIC 2024 (dermatology), Chest X-ray (radiology), Wildlife Animals (natural photography), and Kvasir-SEG (colonoscopy), using domain-specific prompts such as ``a dermoscopy image of melanoma''. For the downstream evaluation (Figure~\ref{fig:gradient_maps}(d)), we trained linear probes \rev{on frozen features} for a 6-class \rev{COCO classification task} and generated Simple Gradient maps for the resulting classifiers.

\textbf{Results.} \rev{Qualitatively, the AFT model produced clearer, more localised saliency maps than baseline CLIP across the OOD domains, most clearly on ISIC and the chest X-rays, and the improvement persisted for the linear probes.}

\rev{Table~\ref{tab:ood_loc} quantifies this, scoring localisation against ground-truth masks in three settings of Figure~\ref{fig:gradient_maps}(c,d): zero-shot on ISIC and Kvasir-SEG, and the 6-class COCO linear probe (1,000 samples per setting, protocol of Section~\ref{sec:claim1}; full five-metric tables in Appendix~\ref{appendix:ood_loc_full}). The vanilla-CLIP row gives the direct quantitative test of Claim~3. AFT roughly quadruples Simple Gradient PG on ISIC (0.107 to 0.428) and more than doubles it on the COCO probe (0.140 to 0.333). On Kvasir-SEG every ImageNet-fine-tuned arm stays near the floor (Simple Gradient PG at or below 0.072), so the small polyp targets provide little evidence either way. The transfer evidence is thus strong in two of three settings and absent in the third. Our plain-AFT reproductions match the released checkpoint in all three settings, extending the model-level reproducibility of Section~\ref{sec:claim1} to OOD and downstream data. The saliency-guided extension of Section~\ref{sec:sgr} also transfers: Simple Gradient and Grad-ECLIP improve monotonically with the EPG weight on every dataset (ISIC Simple Gradient PG 0.428 to 0.820), even though Grad-ECLIP is never optimised by the EPG loss, and the EPG models are the only ones that move off the Kvasir-SEG floor (Simple Gradient PG up to 0.265). Grad-CAM responds to the dose only on ISIC and stays flat in the other settings, consistent with its instability in Section~\ref{sec:claim1}. The EPG rows build on AFT$^4$ (ours), so they isolate the extension's effect on top of AFT; being COCO-fine-tuned, they hold an in-domain advantage on the COCO probe.}

\begin{table}[h]
\centering
\caption{\rev{Localisation against ground-truth masks in the transfer settings of Figure~\ref{fig:gradient_maps}(c,d): zero-shot on ISIC and Kvasir-SEG, and the 6-class COCO linear probe. 1,000 samples per setting, seed 42, protocol of Section~\ref{sec:claim1}. SG = Simple Gradient, GE = Grad-ECLIP, GC = Grad-CAM; Grad-ECLIP is not defined for the linear probe, so the COCO columns report Grad-CAM. Full five-metric tables in Appendix~\ref{appendix:ood_loc_full}. Bold marks the best value per column; italics mark the best among the ImageNet-only arms.}}
\label{tab:ood_loc}
\resizebox{0.88\textwidth}{!}{
{\color{revcolor}%
\begin{tabular}{l|ccc|ccc|ccc}
\toprule
& \multicolumn{3}{c|}{ISIC (zero-shot)} & \multicolumn{3}{c|}{Kvasir-SEG (zero-shot)} & \multicolumn{3}{c}{COCO linear probe} \\
Model & SG PG & GE PG & SG IoU & SG PG & GE PG & SG IoU & SG PG & GC PG & SG IoU \\
\midrule
\multicolumn{10}{l}{\textit{ImageNet fine-tuning only}} \\
CLIP (w/o AFT)      & 0.107 & 0.397 & 0.091 & 0.058 & \textit{0.085} & 0.062 & 0.140 & 0.226 & 0.076 \\
AFT$^4$ (original)  & 0.428 & 0.460 & \textit{0.282} & 0.051 & 0.020 & 0.121 & \textit{0.333} & \textbf{\textit{0.244}} & \textit{0.145} \\
AFT$^1$ (ours)      & \textit{0.457} & \textit{0.469} & 0.249 & \textit{0.072} & 0.043 & \textit{0.137} & 0.324 & 0.230 & 0.139 \\
AFT$^4$ (ours)      & 0.424 & 0.426 & 0.274 & 0.056 & 0.028 & 0.127 & 0.328 & 0.239 & 0.144 \\
\midrule
\multicolumn{10}{l}{\textit{COCO fine-tuning with EPG supervision}} \\
\;+ EPG $w_s{=}0.2$ & 0.648 & 0.750 & 0.350 & 0.113 & 0.037 & 0.143 & 0.387 & 0.218 & 0.166 \\
\;+ EPG $w_s{=}0.4$ & 0.763 & 0.852 & 0.383 & 0.180 & 0.086 & 0.166 & 0.430 & 0.226 & 0.181 \\
\;+ EPG $w_s{=}0.6$ & \textbf{0.820} & \textbf{0.905} & \textbf{0.407} & \textbf{0.265} & \textbf{0.187} & \textbf{0.197} & \textbf{0.467} & 0.236 & \textbf{0.193} \\
\bottomrule
\end{tabular}}
}
\end{table}

\subsubsection{Accuracy Trade-off (Claim 4)}
\label{sec:claim4}

The current experiment evaluates whether the gains in interpretability and robustness come at a high cost to the model's standard performance.

\paragraph{Experimental setup.} We compare standard CLIP, the checkpoint provided by the original paper (AFT$^4$ original), and our two models with $\varepsilon = 1/255$ and $\varepsilon = 4/255$, respectively. We evaluate the models' zero-shot accuracy on 1,000 random samples per dataset, drawn from both in-distribution and out-of-distribution datasets. \rev{The subset is selected with a fixed random seed (42), so all models are evaluated on the same reproducible sample. Classification uses the fixed prompt template ``A photo of a \{\}.'' with each dataset's default class names, encoded once by the frozen text encoder; zero-shot logits are cosine similarities between image and class-text embeddings, scaled by CLIP's learned logit scale. Adversarial evaluation follows the AutoAttack recipe \citep{apgd} with two-stage APGD (untargeted cross-entropy, then targeted DLR on surviving samples) under $L_{\infty}$ budgets of $2/255$ and $4/255$; the full attack configuration is in Appendix~\ref{appendix:seed_cis}.}

\begin{table}[ht]
\centering
\caption{Adversarial Robustness and Accuracy Results. (ours) denotes retrained models; (original) denotes original checkpoints provided by authors. \rev{Imagenette values are means over five evaluation seeds; all other columns use the fixed seed 42, except the COCO AFT ($w_s{=}0$) row, which reports five-seed means throughout. Per-cell 95\% confidence intervals for every dataset are given in Appendix~\ref{appendix:seed_cis}.}}
\label{tab:reproduction_results_complete}
\resizebox{0.72\textwidth}{!}{
\begin{tabular}{llcccccc}
\toprule
$L_\infty$ & \textbf{Model} & \textbf{CIFAR-10} & \textbf{STL-10} & \textbf{EuroSAT} & \textbf{FGVC Air.} & \textbf{Imagenette} & \textbf{Average} \\
\midrule

\multirow{5}{*}{\rotatebox[origin=c]{90}{clean}} 
& CLIP               & \rev{\textbf{89.4}} & \rev{\textbf{98.1}} & \rev{\textbf{40.5}} & \rev{\textbf{22.0}} & \rev{\textbf{99.58}} & \rev{\textbf{69.92}} \\
& AFT$^1$ (ours)     & \rev{86.4} & \rev{97.7} & \rev{9.5}   & \rev{21.4} & \rev{99.26} & \rev{62.85} \\
& AFT$^4$ (ours)     & \rev{66.2} & \rev{91.4} & \rev{10.5}  & \rev{12.2} & \rev{98.60} & \rev{55.78} \\
& AFT$^4$ (original) & \rev{63.2} & \rev{90.4} & \rev{11.9} & \rev{12.0} & \rev{98.22} & \rev{55.14} \\
& COCO AFT ($w_s{=}0$) & \rev{55.7} & \rev{83.4} & \rev{8.3} & \rev{13.7} & \rev{95.20} & \rev{51.26} \\

\midrule

\multirow{5}{*}{\rotatebox[origin=c]{90}{2/255}} 
& CLIP               & \rev{0.0}  & \rev{0.0}  & \rev{0.0}  & \rev{0.0} & \rev{0.00} & \rev{0.00} \\
& AFT$^1$ (ours)     & \rev{20.3} & \rev{56.1} & \rev{0.6}  & \rev{1.8} & \rev{83.38} & \rev{32.44} \\
& AFT$^4$ (ours)     & \rev{\textbf{44.9}} & \rev{\textbf{79.6}} & \rev{10.1} & \rev{3.2} & \rev{\textbf{94.52}} & \rev{\textbf{46.46}} \\
& AFT$^4$ (original) & \rev{41.6} & \rev{77.2} & \rev{\textbf{11.9}} & \rev{\textbf{3.3}} & \rev{94.08} & \rev{45.62} \\
& COCO AFT ($w_s{=}0$) & \rev{35.4} & \rev{64.7} & \rev{2.2} & \rev{4.5} & \rev{83.66} & \rev{38.07} \\

\midrule

\multirow{5}{*}{\rotatebox[origin=c]{90}{4/255}} 
& CLIP               & \rev{0.0}  & \rev{0.0}  & \rev{0.0}  & \rev{0.0} & \rev{0.00} & \rev{0.00} \\
& AFT$^1$ (ours)     & \rev{0.2}  & \rev{6.1}  & \rev{0.0}  & \rev{0.0} & \rev{23.32} & \rev{5.92} \\
& AFT$^4$ (ours)     & \rev{\textbf{24.6}} & \rev{\textbf{58.4}} & \rev{10.0} & \rev{1.1} & \rev{\textbf{86.14}} & \rev{\textbf{36.05}} \\
& AFT$^4$ (original) & \rev{22.4} & \rev{57.4} & \rev{\textbf{11.7}} & \rev{\textbf{1.2}} & \rev{84.54} & \rev{35.45} \\
& COCO AFT ($w_s{=}0$) & \rev{15.3} & \rev{38.2} & \rev{0.6} & \rev{1.2} & \rev{59.64} & \rev{22.99} \\

\bottomrule
\end{tabular}
}
\end{table}






\paragraph{Reproducibility.} \rev{We repeated the evaluation with five seeds, drawing a fresh 1,000-sample subset each time; the attack is deterministic, so the spread is pure sample-selection variance. The 95\% confidence intervals stay within $\pm 2.42$ accuracy points across all checkpoint, dataset, and budget combinations (Appendix~\ref{appendix:seed_cis}), so the fixed-seed values above are stable and the trend holds under every seed.}

\rev{At $\varepsilon=1/255$ the clean-accuracy cost matches the original report on CIFAR-10, STL-10, and Imagenette (at most 3 points), but the average drop is 7.1 points, driven by EuroSAT, where all AFT models fall to the 10-class chance level on every seed: zero-shot transfer to satellite imagery does not survive ImageNet fine-tuning. We consider Claim~4 reproduced on natural-image benchmarks, while the ``reasonable accuracy'' framing does not extend to every domain.}

\subsubsection{Remove and Retrain}
\label{sec:roar}
To prove that saliency maps after AFT highlight the task-related regions of the input image, the authors of the original paper use the Remove and Retrain (ROAR) method \citep{hooker2019benchmark}. ROAR measures how much the actual information content of the dataset decreases by forcing a new model to learn from the modified masked data.

\paragraph{Experimental setup.} First, we generate static saliency maps for the entire training and validation sets of Caltech256 \citep{caltech} and Imagenette \citep{Imagenette}, once with the reference AFT$^4$ model and once with standard CLIP. We then construct multiple masked versions of each dataset by removing the top $k \in \{25\%, 50\%, 75\%, 95\%\}$ most salient pixels, as identified by the saliency method. \rev{Removal uses a zero fill: the image is multiplied element-wise with a binary mask that zeroes the selected pixels.} For each masking level and saliency source, a ResNet-50 is randomly initialised and trained from scratch on the corresponding masked training set, \rev{using the cross-entropy loss and Adam with a learning rate of $10^{-3}$, a batch size of 25, and 10 epochs. Inputs are resized and centre-cropped to $224 \times 224$; no data augmentation is applied. We report the retrained network's classification accuracy on the correspondingly masked validation set.}

\begin{figure}[t]
    \centering
    \begin{subfigure}[t]{0.24\textwidth}
        \centering
        \includegraphics[width=\linewidth]{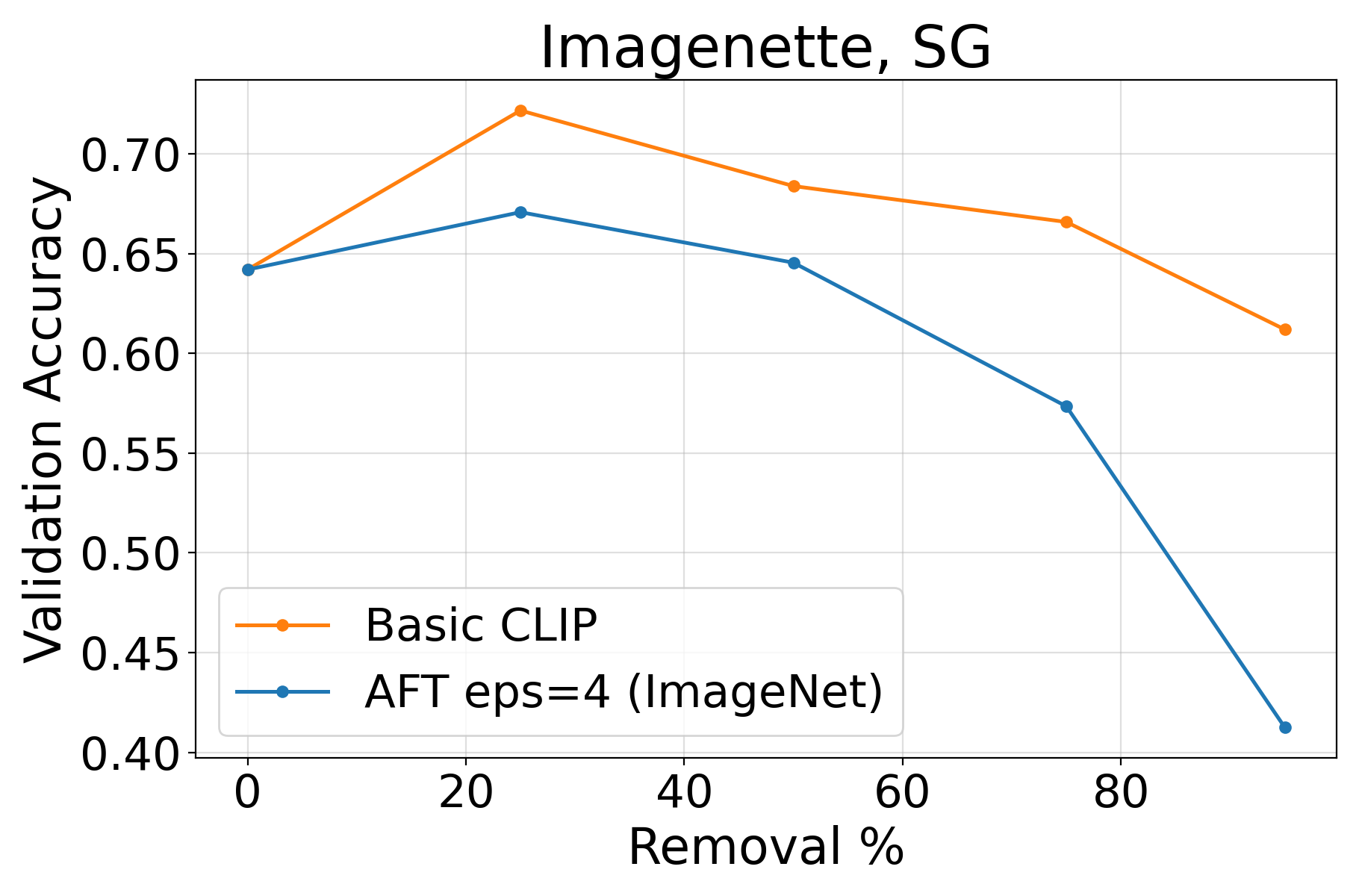}
        \caption{}
    \end{subfigure}
    \hfill
    \begin{subfigure}[t]{0.24\textwidth}
        \centering
        \includegraphics[width=\linewidth]{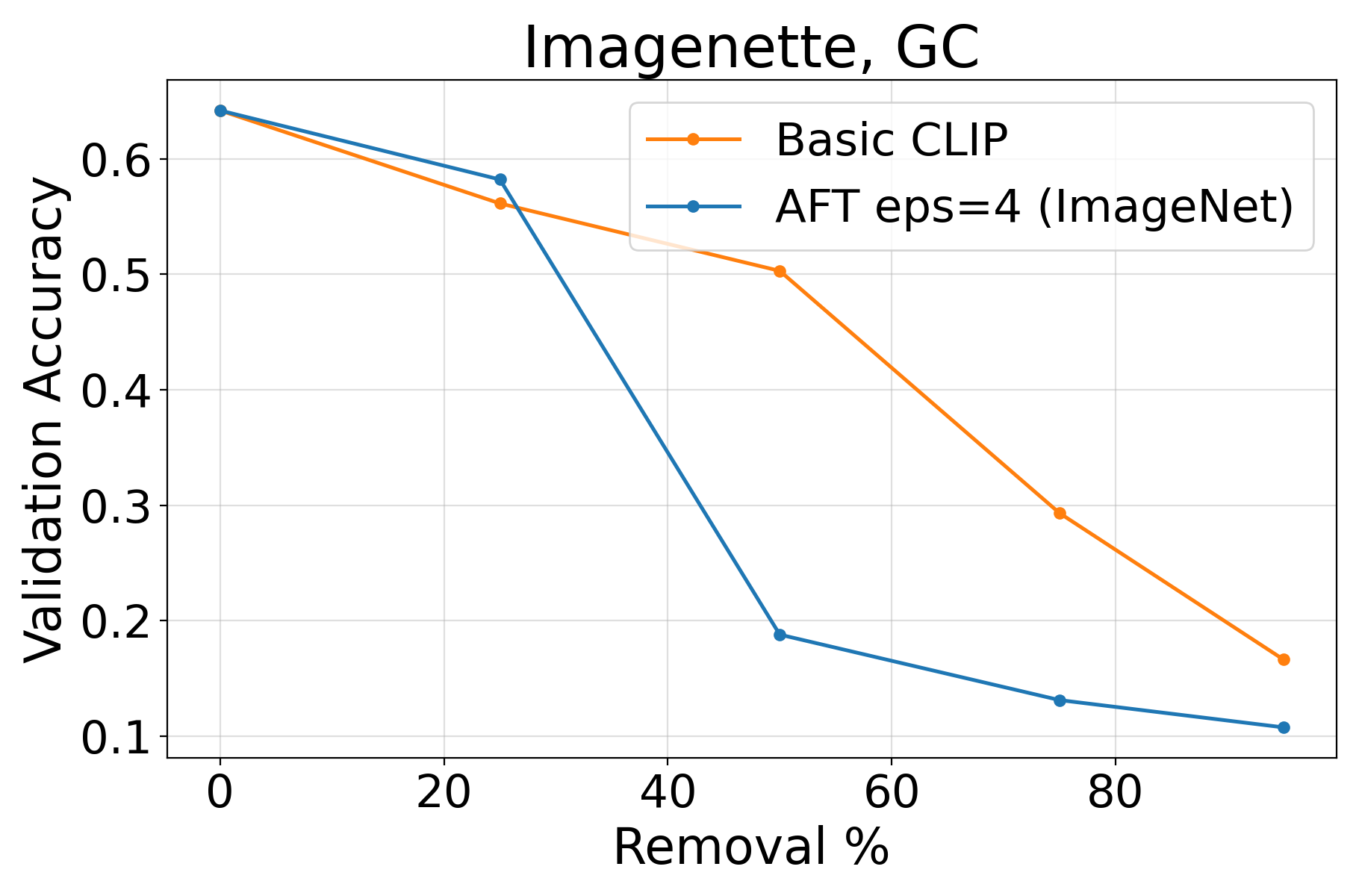}
        \caption{}
    \end{subfigure}
    \hfill
    \begin{subfigure}[t]{0.24\textwidth}
        \centering
        \includegraphics[width=\linewidth]{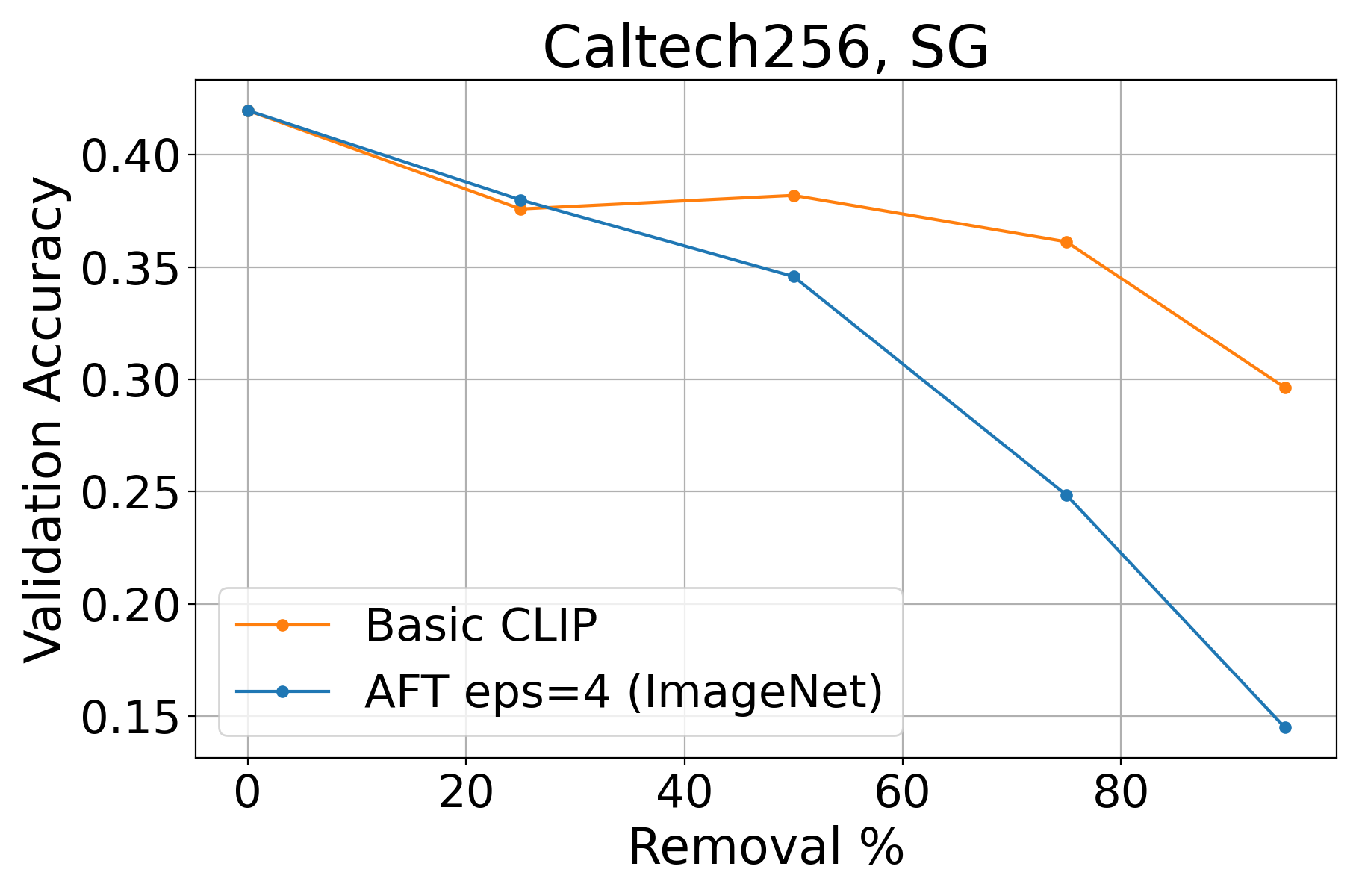}
        \caption{}
    \end{subfigure}
    \hfill
    \begin{subfigure}[t]{0.24\textwidth}
        \centering
        \includegraphics[width=\linewidth]{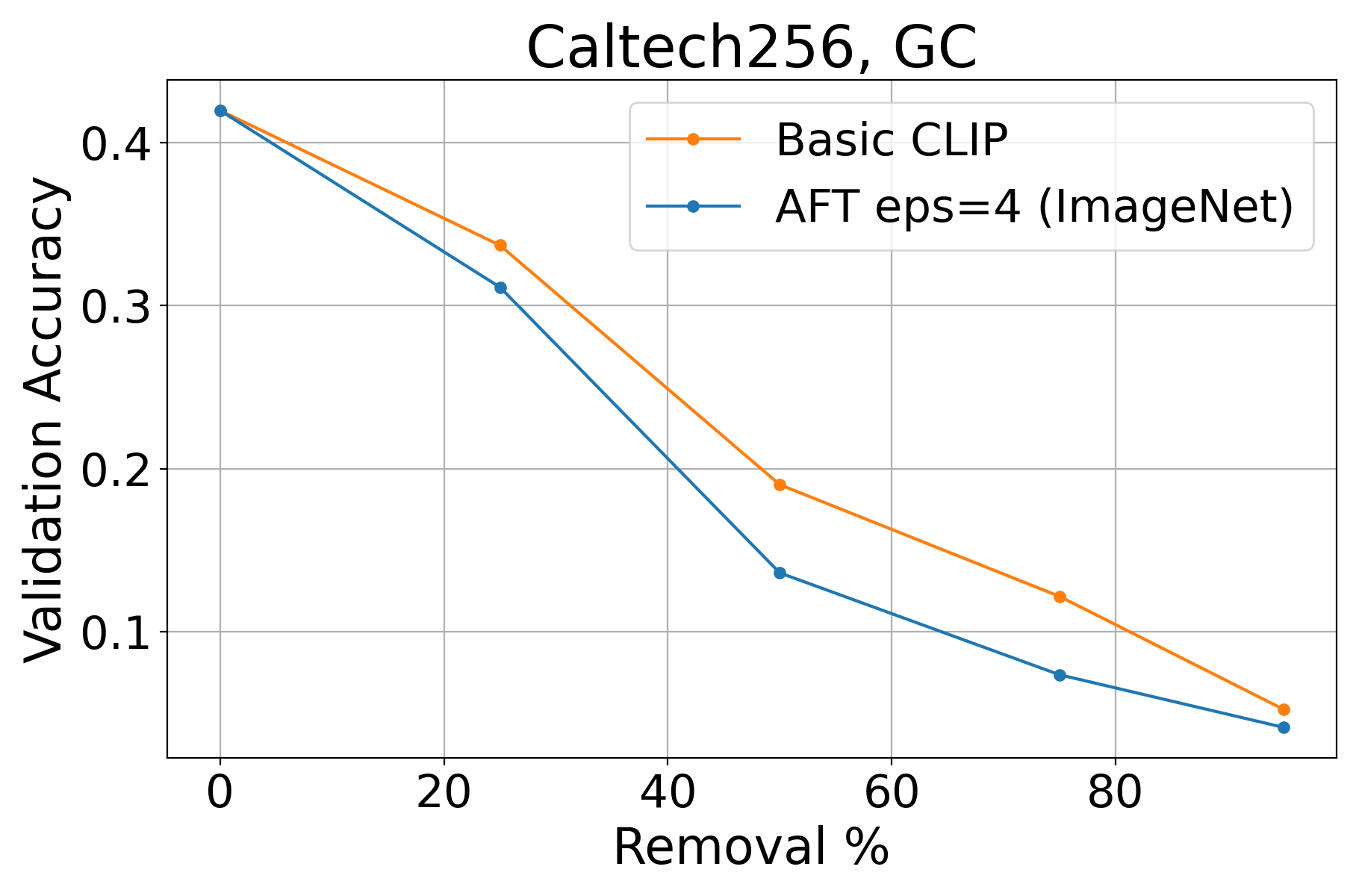}
        \caption{}
    \end{subfigure}

    \caption{\rev{ROAR on Imagenette (a: Simple Gradient, b: Grad-CAM) and Caltech256 (c: Simple Gradient, d: Grad-CAM). Each panel compares saliency maps from the reference AFT$^4$ model against maps from standard CLIP; curves show the validation accuracy of a ResNet-50 retrained from scratch at each removal level. Imagenette curves are means over three seeds; Caltech256 uses a single run.}}
    \label{fig:roar}
\end{figure}

\paragraph{\rev{Results.}} \rev{As illustrated in Figure~\ref{fig:roar}, saliency maps from the AFT reference model produce a steeper accuracy drop under ROAR than maps from standard CLIP in all four settings: removing the pixels AFT marks as salient destroys more task-relevant information, the expected outcome when a saliency method locates the features that drive classification.}

\rev{The effect is strongest in the Imagenette Grad-CAM panel (b), where accuracy falls 53.4 points from 0\% to 95\% removal against 22.9 for Simple Gradient; in panel (a) the AFT curve lies below the standard-CLIP curve at every removal level, and in panel (b) from the 50\% level onward.}

\rev{ROAR therefore supports the causal-relevance part of Claim~1 in all four settings we evaluate. The remaining difference from the original plots is quantitative and most plausibly follows from evaluation details the original paper does not specify.}


\subsection{Extensions to the original paper}
\subsubsection{Cross-Modal Retrieval Robustness on Flickr30k}

\paragraph{Experimental Setup.} To extend the evaluation beyond the benchmarks of the original paper, we evaluate cross-modal retrieval on Flickr30k, measuring Recall@1 and Recall@5 in both directions under clean conditions and under perturbations with $\varepsilon \in \{2/255, 4/255\}$. \rev{We compare baseline CLIP with the AFT models, reporting both the author-provided AFT$^4$ checkpoint and our retrained one} (Table~\ref{tab:flickr30k_retrieval_formatted}).

\textbf{Clean Performance ($\varepsilon = 0$).}
Under clean conditions, our retrained $\mathrm{AFT}^1$ model improves text-to-image retrieval performance compared to the CLIP baseline, achieving higher Recall@k scores. \rev{This is consistent with the known regularising effect of mild adversarial perturbations \citep{xie2020adversarial, salman2020adversarially}. The improvement is direction-dependent: image-to-text recall decreases slightly under AFT$^1$ (R@1 44.41 to 43.02), plausibly because AFT reshapes the image-embedding geometry while the text encoder stays frozen.}

\textbf{Adversarial Robustness ($\varepsilon = 2/255, 4/255$).} Under attack, standard CLIP collapses to near-zero recall already at $\varepsilon = 2/255$, while the AFT models retain substantial retrieval capability: \rev{AFT$^1$ offers the better balance at $\varepsilon = 2/255$, and AFT$^4$ clearly wins under $\varepsilon = 4/255$.}

\textbf{Reproducibility.} The close agreement between the performance of our retrained $\mathrm{AFT}^4$ model and the publicly available checkpoint demonstrates that the robustness gains reported in the original paper are reproducible.

\subsubsection{Saliency-Guided Regularization}
\label{sec:sgr}
\textbf{Experimental setup.} We evaluate the fine-tuning strategy of Section~\ref{sec:methepg}, \rev{trained on COCO 2017 (train2017) with the EPG loss, under the zero-shot ImageNet-S localisation protocol of Section~\ref{sec:claim1}}, using Simple Gradient, Grad-CAM, and Grad-ECLIP saliency maps and sweeping \(w_s \in \{0.2, 0.4, 0.6\}\). We report the same five localisation metrics as Section~\ref{sec:claim1} and evaluate adversarial robustness on the five zero-shot classification benchmarks of Section~\ref{sec:claim4}.

\textbf{Saliency map quality.} Table~\ref{tab:saliency_coco} reports localisation metrics on ImageNet-S across the four settings (\(w_s = 0\), the COCO-fine-tuned control without saliency loss, and \(w_s \in \{0.2, 0.4, 0.6\}\)). For Simple Gradient, every metric improves monotonically with \(w_s\), with the strongest gain on the Pointing Game \rev{(\(+36\%\) relative at \(w_s = 0.6\)). Grad-CAM responds more modestly: its Pointing Game rises from 0.556 to 0.614 over the sweep while its remaining metrics move little, so the dose reaches Grad-CAM's argmax localisation but not its mass distribution.}

\textbf{Transfer to Grad-ECLIP.} Notably, the saliency loss also improves Grad-ECLIP despite never optimising it directly: all five localisation metrics improve monotonically with \(w_s\). This indicates that the saliency loss does not merely overfit the Simple Gradient pathway but induces a structural change in how the encoder distributes gradient information across the input. This cross-method transfer provides some of the strongest evidence that the saliency regularisation improves the encoder's internal representations. A qualitative counterpart to Figure~\ref{fig:gradient_maps}, showing the same image--prompt pairs under saliency-guided regularisation at $w_s = 0.6$, is provided in Appendix~\ref{appendix:saliency_qualitative}; the accompanying repository additionally contains the equivalent figures for $w_s \in \{0.2, 0.4\}$.

\begin{figure}[h]
\centering
\begin{minipage}[t]{0.58\textwidth}
\centering
\vspace{0pt}
\small
\setlength{\tabcolsep}{3pt}
\captionof{table}{Localisation metrics on the ImageNet-S validation set (12,419 images, protocol of Section~\ref{sec:claim1}) for three saliency methods (SG, GC, Grad-ECLIP). \(w_s{=}0\): control fine-tuned on COCO with the identical recipe but without the saliency loss. Remaining rows increase the saliency-loss weight \(w_s\) over the same converged robust checkpoint.}
\label{tab:saliency_coco}
\setlength{\tabcolsep}{2.5pt}%
\begin{tabular}{llccccc}
\toprule
$w_s$ & Method & PG\(\uparrow\) & E-PG\(\uparrow\) & AP\(\uparrow\) & PA\(\uparrow\) & IoU\(\uparrow\) \\
\midrule
\multirow{3}{*}{$w_s=0$}
 & SG      & \rev{0.5579} & \rev{0.4508} & \rev{0.4944} & \rev{0.6866} & \rev{0.2426} \\
 & GC      & \rev{0.5560} & \rev{0.4680} & \rev{0.4977} & \rev{\textbf{0.6568}} & \rev{0.2364} \\
 & Grad-ECLIP  & \rev{0.6547} & \rev{0.5189} & \rev{0.6309} & \rev{0.7257} & \rev{0.3211} \\
\midrule
\multirow{3}{*}{\(w_s = 0.2\)}
 & SG      & \rev{0.6527} & \rev{0.5056} & \rev{0.5506} & \rev{0.7058} & \rev{0.2753} \\
 & GC      & \rev{0.5699} & \rev{0.4675} & \rev{0.4906} & \rev{0.6479} & \rev{0.2302} \\
 & Grad-ECLIP  & \rev{0.7231} & \rev{0.5347} & \rev{0.6574} & \rev{0.7342} & \rev{0.3355} \\
\midrule
\multirow{3}{*}{\(w_s = 0.4\)}
 & SG      & \rev{0.7144} & \rev{0.5444} & \rev{0.5891} & \rev{0.7183} & \rev{0.2976} \\
 & GC      & \rev{0.5831} & \rev{0.4692} & \rev{0.4905} & \rev{0.6484} & \rev{0.2300} \\
 & Grad-ECLIP  & \rev{0.7703} & \rev{0.5489} & \rev{0.6770} & \rev{0.7405} & \rev{0.3461} \\
\midrule
\multirow{3}{*}{\(w_s = 0.6\)}
 & SG      & \rev{\textbf{0.7604}} & \rev{\textbf{0.5776}} & \rev{\textbf{0.6190}} & \rev{\textbf{0.7276}} & \rev{\textbf{0.3145}} \\
 & GC      & \rev{\textbf{0.6137}} & \rev{\textbf{0.4915}} & \rev{\textbf{0.5016}} & \rev{0.6424} & \rev{\textbf{0.2365}} \\
 & Grad-ECLIP  & \rev{\textbf{0.7981}} & \rev{\textbf{0.5591}} & \rev{\textbf{0.6884}} & \rev{\textbf{0.7443}} & \rev{\textbf{0.3525}} \\
\bottomrule
\end{tabular}
\end{minipage}
\hfill
\begin{minipage}[t]{0.40\textwidth}
\centering
\vspace{0pt}
\includegraphics[width=\textwidth]{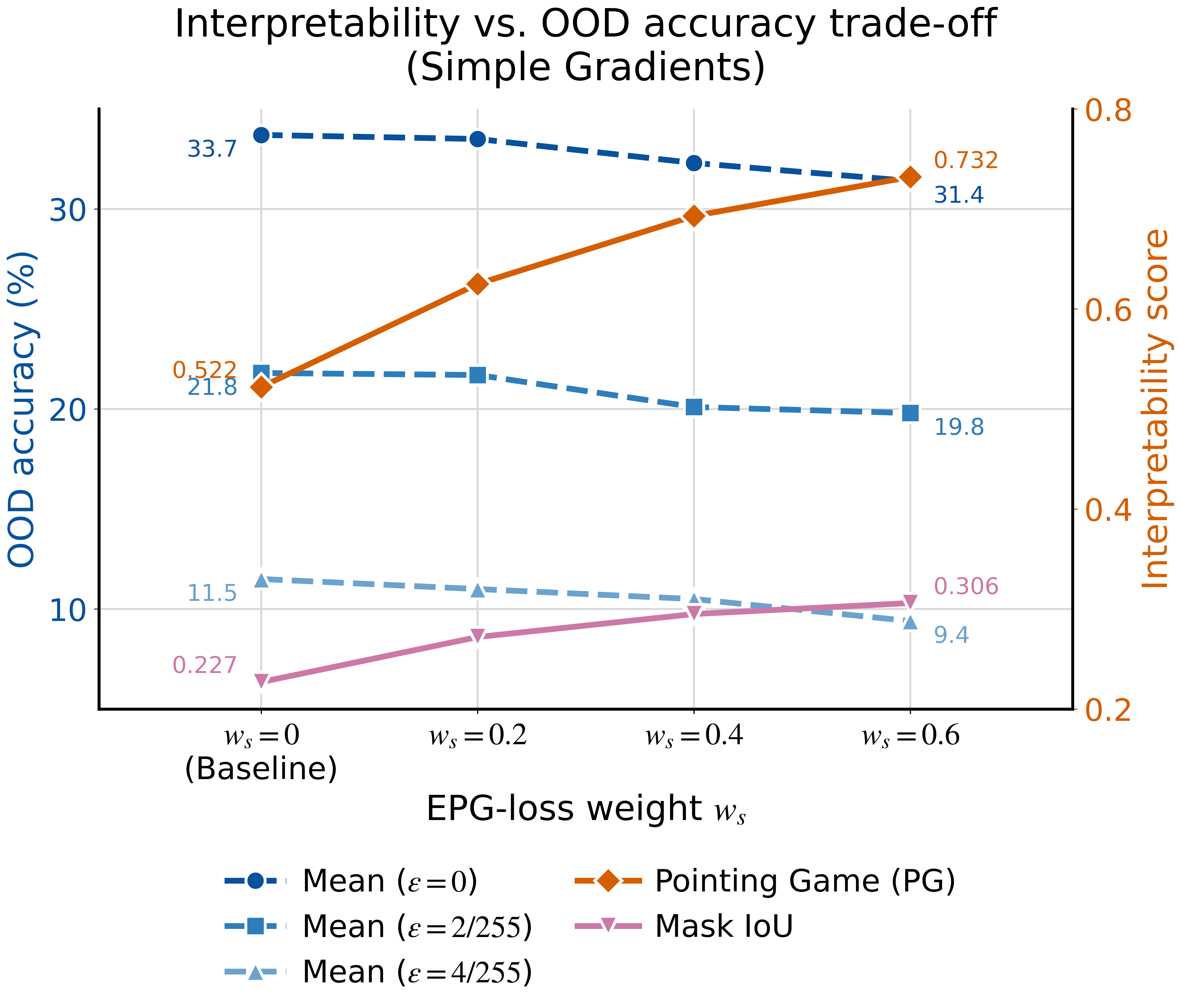}
\captionof{figure}{\rev{Interpretability--accuracy trade-off across EPG-loss weights $w_s$. Solid lines show Simple Gradient Pointing Game and Mask IoU on ImageNet-S; dashed lines show mean zero-shot accuracy across five benchmarks under $\varepsilon \in \{0, 2/255, 4/255\}$. The $w_s=0$ setting is the matched COCO-fine-tuned control.}}
\label{fig:saliency_tradeoff}
\end{minipage}
\end{figure}

\textbf{Interpretability--accuracy trade-off.} Figure~\ref{fig:saliency_tradeoff} jointly visualises the evolution of interpretability and
zero-shot accuracy as $w_s$ increases. At $w_s=0.2$, Simple Gradient
Pointing Game and Mask IoU improve by 9.5 and 3.3 percentage points,
respectively, over the matched $w_s=0$ control, while mean clean,
$2/255$, and $4/255$ zero-shot accuracy decrease by only 1.3, 1.0,
and 0.5 points. At $w_s=0.6$, the corresponding interpretability gains
reach 20.3 and 7.2 points, while the three accuracy reductions remain
between 3.3 and 3.9 points. The sweep therefore exposes a favourable,
tunable trade-off: low EPG weights substantially improve foreground
alignment with little loss in clean or adversarial accuracy, whereas
larger weights provide stronger localisation at a moderate accuracy
cost.


\textbf{Effect on concept alignment.} To test whether the saliency extension affects the encoder's internal representations beyond gradient structure, we repeat the Network Dissection and concept-alignment evaluations of Section~\ref{sec:concept} for all EPG-fine-tuned models \rev{and a matched $w_s{=}0$ control. Detector counts rise monotonically with the dose (326, 344, 363 of 768 at IoU 0.04 versus 320; Table~\ref{tab:nd_iou_sweep}), as does mean Soft-WPMI alignment (0.2808, 0.2845, 0.2861 versus 0.2781)}: the EPG loss modifies where gradient mass is concentrated \rev{while leaving semantic structure intact, with, if anything, a slight gain.}


\rev{\textbf{Faithfulness metrics decoupled from the training objective.} We probe the extension on two faithfulness metrics independent of the mask-overlap objective the EPG loss optimises (Appendix~\ref{appendix:epg_faithfulness}). Under ROAR, removing EPG-selected pixels degrades accuracy at or below the recipe-matched $w_s{=}0$ control in all four settings for Simple Gradient, and for Grad-CAM up to the highest removal levels, where the curves cross as both arms approach the accuracy floor; the margins are modest but consistent, supporting improved faithfulness on a retraining-based metric not aligned with the training objective. Under deletion and insertion \citep{petsiuk2018rise}, however, the separation from a random-ordering control does not grow with the dose and shrinks on the ImageNet validation set, so the deletion/insertion evidence is weaker and less consistent than the ROAR result. A side finding supports the Grad-CAM analysis of Section~\ref{sec:claim1}: under a zeros deletion baseline, Grad-CAM deletion AUC sits above the random control for all six checkpoints. Overall, the extension's clearest measurable effect remains on where saliency mass lands (Tables~\ref{tab:saliency_coco} and \ref{tab:ood_loc}), with supporting but mixed evidence from decoupled faithfulness metrics.}

\rev{\textbf{Prompt-specific localisation.} On 1,838 COCO val2017 images with two prompted object categories each, we generate one saliency map per prompt and measure how much saliency mass follows the prompted object (Appendix~\ref{appendix:epg_faithfulness}). The metric passes its sanity check: Grad-ECLIP is strongly prompt-specific on the AFT baseline (mass ratio 0.73 against chance 0.5) while raw Simple Gradient sits near chance. The EPG dose slightly reduces prompt discrimination. The one quantity it raises, the Simple Gradient pointing game (0.124 to 0.293), only checks that the hottest pixel lands inside the prompted mask, so it reflects sharper foreground pointing, not discrimination. We do not claim the extension improves prompt-specific interpretability.}

\rev{\textbf{Background-restricted adversarial attack.} Proposition~\ref{prop:bg-adv} concerns perturbations confined to the image background. It states that suppressing the feature displacement induced by such perturbations is equivalent, to leading order, to minimising the EPG loss. An EPG-trained encoder should therefore be resistant to background-restricted attacks. We test this prediction with a PGD attack restricted to the background (Table~\ref{tab:bgattack}); the proposition and its proof are given in Appendix~\ref{sec:theory-epg}. For each model we compare the feature damage caused by background perturbations to that caused by foreground perturbations. The plain-AFT baseline sits at 1.25--1.29, so its features are more sensitive to the background than to the object. Every EPG model falls below one, and the ratio decreases monotonically with the EPG dose (0.94, 0.79, 0.68 at $\varepsilon = 2/255$), which is exactly the invariance the proposition predicts. }
\rev{Consistent with the theory, EPG redistributes the encoder's sensitivity toward the foreground; this is a targeted reweighting of where the model is sensitive rather than a uniform reduction in adversarial vulnerability, and it comes with the same modest clean- and adversarial-accuracy trade-off
documented above.}



\textbf{Results.} \rev{The saliency-guided loss improves foreground alignment most clearly} for Simple Gradient, the attribution method it directly optimises, and this improvement partially transfers to Grad-ECLIP. \rev{Grad-CAM shows a smaller, mostly Pointing-Game response to the EPG weight, which we read with caution given the instability observed for CLIP-style encoders in Section~\ref{sec:claim1}. The gains are strongest on mask-based localisation metrics, especially Pointing Game, Energy Pointing Game, AP, and IoU, while the effect on concept alignment is small and does not indicate a loss of semantic structure.}


\rev{The decoupled evaluations qualify these gains: ROAR indicates that EPG-selected pixels are more task-relevant (though the effect narrows at high removal), while deletion/insertion and prompt-specific localisation show no consistent improvement. The background-restricted attack confirms Proposition~\ref{prop:bg-adv} only relatively: EPG shifts sensitivity toward the foreground without a net robustness gain. The extension is thus best seen as a targeted way to concentrate
saliency mass on foreground objects, at a modest, tunable cost in clean
and adversarial zero-shot accuracy.}

\section{Discussion and Conclusion}





This reproduction study provides directional support for the main findings \rev{of \citet{gong2025boosting}. Under our independently implemented evaluation pipeline, AFT improves Simple Gradient saliency localisation, increases neuron-concept alignment, and provides substantially better adversarial robustness than standard CLIP, although these gains come with a reduction in clean zero-shot accuracy. Our retrained checkpoints closely match the author-released checkpoints, which supports model-level reproducibility.}

\rev{Table~\ref{tab:claim_summary} summarises the outcome for each claim. The strongest evidence concerns the Simple Gradient component of the interpretability claim; standard Grad-CAM remains unstable for CLIP-style ViT encoders even though it behaves normally in our CNN sanity check, so we treat its results as inconclusive rather than as decisive evidence for or against AFT.}


\begin{table}[h]
\color{revcolor}
\centering
\caption{Reproduction outcome for the four claims of \citet{gong2025boosting}. References given in text.}
\label{tab:claim_summary}
\footnotesize
\setlength{\tabcolsep}{4pt}
\begin{tabular}{@{}>{\raggedright\arraybackslash}p{0.21\textwidth}>{\raggedright\arraybackslash}p{0.14\textwidth}>{\raggedright\arraybackslash}p{0.51\textwidth}@{}}
\toprule
Claim & Verdict & Primary source of discrepancy \\
\midrule
1a: Simple Gradient & Directional & Protocol choices shift absolute values; none reorder models \\
1b: Grad-CAM & Inconclusive & Unstable on CLIP ViT-B/16 ($14\times14$ grid); metrics still improve \\
2: Concept alignment & Trend & Detector counts set by undocumented IoU threshold ($9\times$ range) \\
3: Transferability & Directional & Quantitative in 2/3 settings; Kvasir-SEG uninformative \\
4: Accuracy trade-off & Reproduced$^\ast$ & EuroSAT at chance for all AFT models \\
\bottomrule
\end{tabular}
\\[2pt]
{\footnotesize $^\ast$ on natural-image benchmarks.}
\end{table}

\rev{Our saliency-guided EPG extension shows that explicit foreground-mask supervision can further improve mask-based localisation metrics and partially transfer to Grad-ECLIP. }
\rev{ROAR provides modest supporting evidence that the highlighted pixels are more task-relevant, with the clearest margins for Simple Gradient and small crossovers for Grad-CAM at high removal fractions, and the background-restricted adversarial evaluation confirms the invariance predicted by our theoretical analysis. At the same time, deletion/insertion and prompt-specific localisation do not improve consistently with the saliency weight, so the extension is best understood as a targeted improvement in foreground alignment rather than a broad gain in attribution faithfulness. Overall, AFT appears reproducible in its main qualitative trends, while exact numerical reproduction and stronger claims about faithfulness require careful attention to protocol details and saliency-method limitations.}

\bibliography{references}
\bibliographystyle{tmlr}

\appendix

\section{Datasets}
\label{appendix:datasets}

\begin{table}[H]
\caption{Datasets used in our reproduction study.}
\label{tab:datasets}
\centering
\begin{tabular}{lll}
\toprule
\textbf{Dataset} & \textbf{Usage} & \textbf{\#images} \\
\midrule
ImageNet 2012 & Adversarial fine-tuning & 1.28M \\
COCO 2017 & Saliency-guided fine-tuning, linear probing, EPG evaluations & 118K \\
ImageNet-S & Saliency localisation evaluation & 12.4K \\
Caltech256 & ROAR analysis & 30.6K \\
Flickr30k & Retrieval evaluation & 31K \\
ISIC 2024 & OOD saliency evaluation & 401K \\
Chest X-ray & OOD saliency evaluation & 5.9K \\
Wildlife Animals & OOD saliency evaluation & 5.2K \\
Kvasir-SEG & OOD saliency evaluation & 1K \\
Broden & Neuron-concept alignment & 63K \\
\midrule
\multicolumn{3}{l}{\textbf{Zero-shot Evaluation Datasets (Validation/Test Splits)}} \\
CIFAR-10 & Zero-shot evaluation & 10K \\
STL-10 & Zero-shot evaluation & 8K \\
EuroSAT & Zero-shot evaluation & 5.4K \\
FGVC Aircraft & Zero-shot evaluation & 3.3K \\
Imagenette & Zero-shot evaluation, ROAR analysis & 3.9K \\
\bottomrule
\end{tabular}
\end{table}

\section{Adversarial Fine-Tuning Formulation}
\label{appendix:AFT}

The original paper formulates AFT as a minimax optimisation problem that enforces robustness while implicitly regularising input gradients.

Given an image $x$, its image embedding $I_x = f_\theta(x)$ (unit-normalised), and corresponding text embedding $T_x$ (unit-normalised), the supervised AFT objective is:
\begin{equation}
\min_\theta \E_{x \sim \mathcal{D}_{\text{train}}} \max_{\delta_x}
\left[ \frac{1}{2} \left( T_x^\top \E_{z \sim \mathcal{N}(0, \sigma^2 I)}[f_\theta(x+z+\delta_x)] - T_x^\top I_x \right)^2 - h(\delta_x) \right]
\label{eq:supervised_aft}
\end{equation}
where $\delta_x$ is the adversarial perturbation, $z$ is Gaussian smoothing noise that improves optimisation stability, and $h(\cdot)$ is a regularisation term that constrains the perturbation magnitude. This formulation corresponds to Equation~(1) in \citet{gong2025boosting}.

\rev{\citet{gong2025boosting} show (their Section~3.1) that a first-order Taylor expansion of the inner maximisation exposes the gradient of the Gaussian-smoothed similarity score, i.e.\ the SmoothGrad map \citep{smilkov2017smoothgrad}, and that Fenchel duality rewrites the objective as an explicit penalty $h^\star$ on that map. Choosing the smooth $\ell_1$ (Huber) form for $h^\star$ \citep{huber1964robust} gives a penalty that is quadratic for small gradient values and linear for large ones, promoting sparse input gradients and hence cleaner attributions. We adopt this formulation without modification. The derivation is given in Section~3.1 of \citet{gong2025boosting}.}


\section{\rev{Qualitative Effect of Saliency-Guided Regularisation}}
\label{appendix:saliency_qualitative}

\rev{Figure~\ref{fig:saliency_qualitative_ws06} complements the quantitative results of Section~\ref{sec:sgr}: for the image--prompt pairs of Figure~\ref{fig:gradient_maps}, it shows Simple Gradient and Grad-ECLIP maps from the AFT$^4$ backbone fine-tuned with saliency-guided regularisation at $w_s = 0.6$. Across the sweep $w_s \in \{0.2, 0.4, 0.6\}$, each increment concentrates saliency mass further on the referenced object and weakens off-object activations, for both attribution methods, matching the monotonic metric trend in Table~\ref{tab:saliency_coco}. The gains between consecutive values are visually small, so we include only $w_s = 0.6$.}

\begin{figure}[h]
\centering
\includegraphics[width=0.62\textwidth]{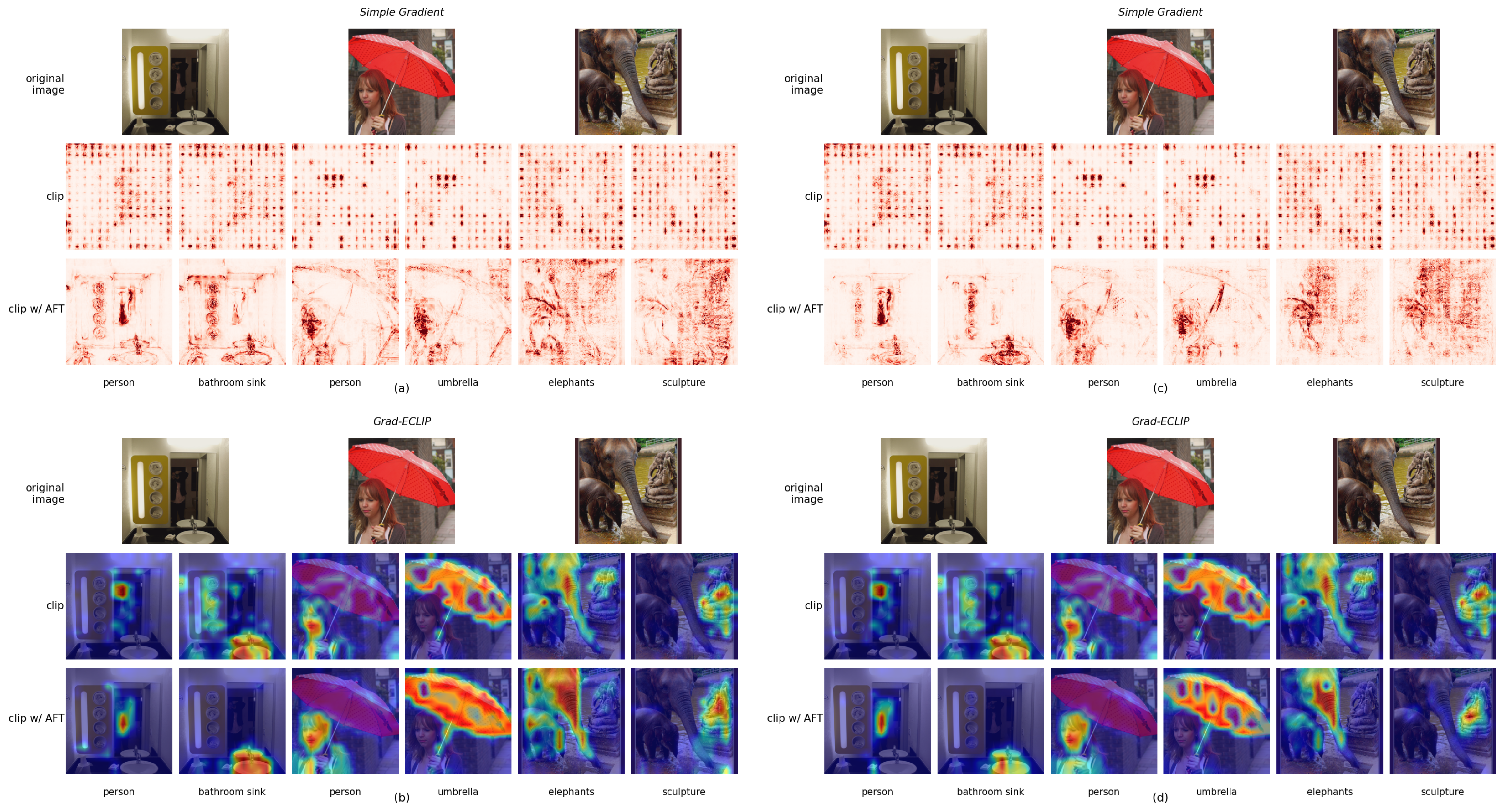}
\caption{Saliency map comparison between CLIP (middle row) and CLIP with AFT (bottom row) for the original AFT$^4$ backbone (left column, as in Figure~\ref{fig:gradient_maps}) and the AFT$^4$ backbone with saliency supervision (right column). Images (a) and (c) are Simple Gradient maps for different text prompts, images (b) and (d) are Grad-ECLIP visualisations for the same images.
Saliency mass is visibly more concentrated on the referenced object and contains fewer off-object artefacts than the corresponding AFT$^4$ maps in Figure~\ref{fig:gradient_maps}.}
\label{fig:saliency_qualitative_ws06}
\end{figure}


\section{\rev{Grad-CAM Sanity Check on a Standard CNN}}
\label{appendix:gradcam_sanity}

\rev{To verify that the Grad-CAM instability reported in Section~\ref{sec:claim1} is attributable to the CLIP ViT architecture rather than to our evaluation pipeline, we ran the identical Grad-CAM procedure on an ImageNet-pretrained ResNet-50 (torchvision) and on CLIP ViT-B/16 (zero-shot, prompt ``A photo of a \{\}.'') for 16 Imagenette validation images covering all ten classes. For the CNN, Grad-CAM is computed on the final convolutional block with respect to the ground-truth class logit; for CLIP we use the same procedure as in the main text.}

\rev{Figure~\ref{fig:gradcam_sanity} shows the resulting maps. ResNet-50 produces smooth heatmaps centred on the target object in essentially every image. CLIP ViT-B/16 is inconsistent: on some images it localises the object well, but on many it collapses into a few blocky high-intensity patches that frequently land on background regions. Quantitatively, ResNet-50 places $0.41 \pm 0.05$ of its saliency mass in the top-20\% pixels, tightly clustered across images; CLIP ViT-B/16 places $0.76 \pm 0.21$ there, with a four-fold larger spread. The failure mode is over-concentration in a few upsampled patches whose placement is unreliable. This is consistent with the class-activation map being computed on the coarse $14\times14$ token grid of ViT-B/16 and supports attributing the Grad-CAM failures to the architecture rather than to our implementation.}

\begin{figure}[h]
\centering
\includegraphics[width=0.62\textwidth]{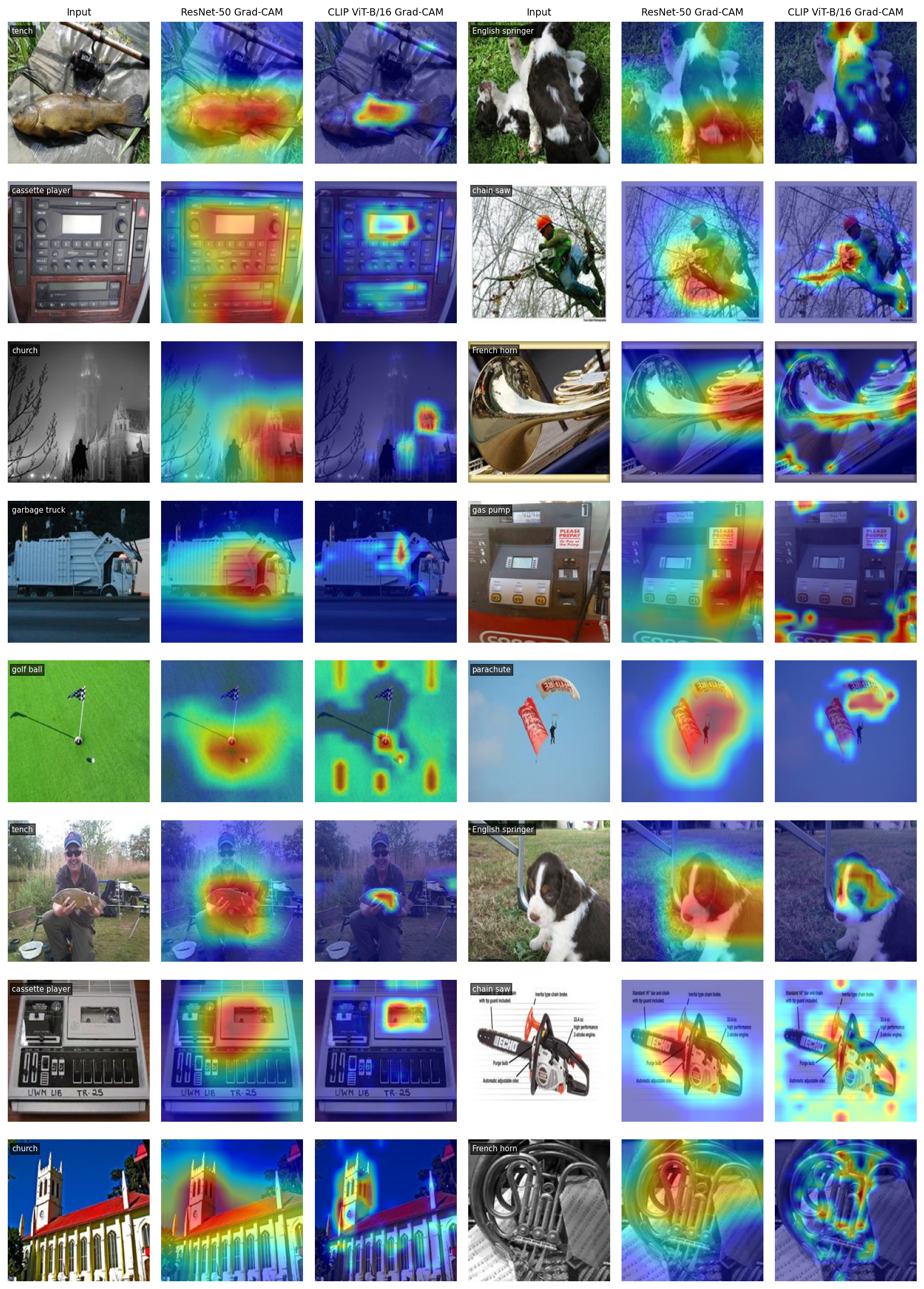}
\caption{\rev{Grad-CAM sanity check. Triplets of input image, ResNet-50 Grad-CAM, and CLIP ViT-B/16 Grad-CAM for 16 Imagenette validation images, computed with the identical pipeline. The CNN yields stable, object-centred maps; CLIP ViT-B/16 yields coarse, fragmented, and spatially unstable maps.}}
\label{fig:gradcam_sanity}
\end{figure}

\section{Environmental Impact}
\label{appendix:environmental_impact}
\rev{We estimated the carbon footprint of our experiments using the CodeCarbon \citep{codecarbon} framework, which provides automated tracking of hardware-level energy consumption and converts it into carbon emissions as $\mathrm{CO_2e} = \mathrm{CI} \cdot \mathrm{PUE} \cdot E$, where $\mathrm{CI}$ denotes the carbon intensity of electricity generation, $\mathrm{PUE}$ is the power usage effectiveness of the computing infrastructure, and $E$ represents the total electrical energy consumed by the CPU, GPU, and memory during the experiments.}

\rev{All reported values correspond to the computational resources required to train the main AFT$^4$ model from scratch, following the experimental protocol described in Section~\ref{subsec:params}.} Over a total runtime of approximately \textbf{34.6 hours}, the experiments consumed \textbf{28.7 kWh} of electricity.

\rev{Using the 2024 average carbon intensity the Netherlands and assuming a PUE of 1.0 (as reported by CodeCarbon), this corresponds to an estimated \textbf{7.7 kg of CO$_2$e emissions}.} Importantly, these estimates are fully reproducible, as they are directly derived from the raw CodeCarbon logs.

\rev{Beyond the initial training, the saliency-guided pipeline (COCO pre-training of the backbone, 5.9 hours, plus four fine-tuning runs, 11.9 hours, on two NVIDIA A100 GPUs (40\,GB each)) consumed \textbf{14.86 kWh} in total, emitting an estimated \textbf{3.98 kg CO$_2$e}; a single fine-tuning run accounts for roughly \textbf{0.66 kg CO$_2$e}.}
\section{Theoretical Properties of the Saliency-Guided Objective}
\label{sec:theory-epg}

\rev{The saliency-guided regularisation of Section~\ref{sec:methepg} admits an adversarial reading that clarifies its relationship to the framework of \citet{gong2025boosting}. Throughout, $s(x;\theta) = I_{\text{orig}}^\top f_\theta(x)$ is the feature consistency score, $g_x = \nabla_x s(x;\theta)$ its input gradient, $S_p = \sum_c |g_{x,p}^{(c)}|$ the unnormalised Simple Gradient attribution at pixel $p$, and $S = \{S_p\}_{p \in \Omega}$ the full spatial map, so that $\|S\|_1 = \sum_{p \in \Omega} S_p$ and $\mathcal{L}_{\text{EPG}} = 1 - \|S\|_1^{-1} \sum_{p \in M} S_p$.}

\paragraph{EPG as a background-restricted adversarial objective.}
\rev{The adversarial reading makes precise what EPG asks the encoder to be invariant to.} Let
\[
\mathcal{A}_{M^c}(\varepsilon)=
\{\delta:\|\delta\|_\infty\le\varepsilon,\;
\delta_p=0\ \forall p\in M\}
\]
denote the set of $\ell_\infty$-bounded perturbations supported \rev{only} on the background
$M^c=\Omega\setminus M$, and consider the worst-case change of the \rev{feature-consistency} score
\begin{equation}
    \Phi_M(x;\theta)
    =
    \max_{\delta\in\mathcal{A}_{M^c}(\varepsilon)}
    \big|s(x+\delta;\theta)-s(x;\theta)\big|.
    \label{eq:bg-adv}
\end{equation}

\begin{proposition}
\label{prop:bg-adv}
To first order in $\varepsilon$,
\begin{equation}
    \Phi_M(x;\theta)
    =
    \varepsilon\sum_{p\notin M}|g_{x,p}|
    +O(\varepsilon^2).
    \label{eq:bg-adv-first-order}
\end{equation}
Consequently, up to the normalising constant $\|S\|_1$ and the perturbation budget $\varepsilon$, minimising $\Phi_M$ over $\theta$ is equivalent, to leading order, to minimising $\mathcal{L}_{\mathrm{EPG}}$.
\end{proposition}

\begin{proof}
Taylor-expanding $s$ around $x$,
\[
s(x+\delta;\theta)-s(x;\theta)
=
\langle g_x,\delta\rangle
+
O(\|\delta\|^2).
\]
Since $\delta_p=0$ for $p\in M$, the inner product reduces to
$\sum_{p\notin M}\langle g_{x,p},\delta_p\rangle$.
Each term is independently maximised over the $\ell_\infty$ ball by
$\delta_p=\varepsilon\,\mathrm{sgn}(g_{x,p})$, yielding
\[
\Phi_M(x;\theta)
=
\varepsilon\sum_{p\notin M}|g_{x,p}|
+
O(\varepsilon^2).
\]
Dividing by $\|S\|_1$ recovers
$\varepsilon\,\mathcal{L}_{\mathrm{EPG}}$.
\end{proof}

\rev{Proposition~\ref{prop:bg-adv} gives the EPG objective a concrete adversarial interpretation: minimising the EPG loss is, to leading order, equivalent to making the encoder invariant to perturbations that are restricted to the image background. This is the same first-order Taylor duality that motivates AFT, but applied to a spatially restricted perturbation set. The proposition therefore yields a directly testable prediction: introducing the EPG objective should progressively reduce the model's sensitivity to background perturbations relative to foreground perturbations.}

\rev{\paragraph{Empirical validation.}
To verify this prediction, we perform region-restricted PGD attacks (100 steps, 1{,}000 samples, $\varepsilon \in \{2/255, 4/255\}$) on the COCO validation set. Following Eq.~\ref{eq:bg-adv}, perturbations are constrained either to the annotated foreground mask or to its complement while keeping the same $\ell_\infty$ budget. We then measure the resulting feature-consistency damage $\Phi$, i.e., the worst-case change in the feature-consistency score under the corresponding perturbation.}

\begin{table}[h]
\centering
\caption{
Feature-consistency damage $\Phi$ under region-restricted PGD attacks on COCO. Background attacks perturb only pixels outside the object mask, while foreground attacks perturb only object pixels. Ratios below one indicate that the representation is less sensitive to background perturbations than to foreground perturbations, as predicted by Proposition~\ref{prop:bg-adv}. The plain-AFT row is our ImageNet-trained AFT$^4$ reproduction; the EPG rows share an identical COCO fine-tuning recipe and differ only in $w_s$.
}
\label{tab:bgattack}
\small
{\color{revcolor}%
\begin{tabular}{lcccc}
\toprule
Model & $\varepsilon$ & Background $\Phi$ & Foreground $\Phi$ & BG/FG \\
\midrule
AFT$^4$ (no EPG) & 2/255 & 0.0036 & 0.0029 & 1.253 \\
+ EPG $w_s{=}0.2$ & 2/255 & 0.0074 & 0.0079 & 0.936 \\
+ EPG $w_s{=}0.4$ & 2/255 & 0.0070 & 0.0088 & 0.787 \\
+ EPG $w_s{=}0.6$ & 2/255 & 0.0066 & 0.0097 & 0.677 \\
\midrule
AFT$^4$ (no EPG) & 4/255 & 0.0123 & 0.0095 & 1.293 \\
+ EPG $w_s{=}0.2$ & 4/255 & 0.0175 & 0.0188 & 0.933 \\
+ EPG $w_s{=}0.4$ & 4/255 & 0.0166 & 0.0212 & 0.784 \\
+ EPG $w_s{=}0.6$ & 4/255 & 0.0160 & 0.0237 & 0.674 \\
\bottomrule
\end{tabular}}
\end{table}

\rev{Table~\ref{tab:bgattack} confirms the theoretical prediction. The plain AFT model is slightly more sensitive to background perturbations than to foreground perturbations (BG/FG $>1$). Introducing the EPG objective consistently reverses this relationship across both perturbation budgets, with the BG/FG ratio falling below one for every EPG model and decreasing monotonically as the saliency weight $w_s$ increases. Thus, stronger saliency supervision progressively shifts the encoder's sensitivity toward foreground regions while suppressing sensitivity to the background, matching the behaviour predicted by Proposition~\ref{prop:bg-adv}. One caveat is that the plain-AFT baseline was never fine-tuned on COCO, so part of the AFT-versus-EPG contrast may reflect COCO exposure rather than the EPG term itself.}

\section{Protocol Sensitivity of the Localisation Metrics}
\label{appendix:protocol_sensitivity}

\rev{Section~\ref{sec:claim1} attributes the shift between our absolute localisation values and those of \citet{gong2025boosting} to undocumented protocol choices. To verify that these choices do not drive the model comparisons, we exposed the previously hardcoded protocol parameters as configuration and re-evaluated our retrained AFT$^4$ checkpoint against the author-released AFT$^4$ checkpoint in every cell of a sweep: binarisation threshold $\text{topk} \in \{0.1, 0.2, 0.3\}$ crossed with pointing-game hit radius $\in \{0, 5, 15\}$ pixels, plus two one-factor cells (bilinear mask downsampling, and clipping the saliency map at the 99th percentile). Each cell uses 3,000 images from the intersection of the ImageNet validation set and ImageNet-S. A non-fine-tuned CLIP arm was evaluated in every cell of the same grid (the clip columns below), so the AFT-versus-baseline gap is also shown per cell.}

\rev{As the metric definitions imply, Pointing Game is invariant to topk and IoU to the hit radius, so the grid collapses to Table~\ref{tab:protocol_pg} and the threshold summary below. Within every saliency method, the sign of the reproduction-minus-released gap is identical in all cells. The AFT-versus-baseline contrast is equally stable: the released AFT$^4$ beats non-fine-tuned CLIP in every cell for Simple Gradient (PG gap $+0.295$ to $+0.396$; $+0.113$ in the degenerate percentile-clipping cell) and Grad-CAM ($+0.325$ to $+0.393$, that cell included), while Grad-ECLIP sits at parity between the two models everywhere ($-0.014$ to $+0.025$).}

\begin{table}[h]
\centering
\caption{\rev{Pointing Game under varying hit radius (invariant to topk). clip = non-fine-tuned CLIP, base = author-released AFT$^4$, ours = our retrained AFT$^4$; gap = ours $-$ base. 3,000 ImageNet-S samples.}}
\label{tab:protocol_pg}
\small
{\color{revcolor}%
\begin{tabular}{c|cccc|cccc|cccc}
\toprule
& \multicolumn{4}{c|}{Simple Gradient} & \multicolumn{4}{c|}{Grad-CAM} & \multicolumn{4}{c}{Grad-ECLIP} \\
radius & clip & base & ours & gap & clip & base & ours & gap & clip & base & ours & gap \\
\midrule
0  & 0.261 & 0.556 & 0.559 & $+0.003$ & 0.270 & 0.595 & 0.571 & $-0.025$ & 0.667 & 0.655 & 0.661 & $+0.006$ \\
5  & 0.294 & 0.660 & 0.663 & $+0.004$ & 0.341 & 0.713 & 0.685 & $-0.028$ & 0.781 & 0.767 & 0.779 & $+0.012$ \\
15 & 0.388 & 0.783 & 0.783 & $+0.000$ & 0.463 & 0.856 & 0.829 & $-0.026$ & 0.908 & 0.902 & 0.906 & $+0.005$ \\
\bottomrule
\end{tabular}}
\end{table}

\rev{The binarisation-threshold sweep behaves identically for Mask IoU. Within every method the ours-minus-released gap keeps its sign and magnitude across topk $\in \{0.1, 0.2, 0.3\}$ (Simple Gradient within $0.001$, Grad-CAM $-0.009$ to $-0.010$, Grad-ECLIP $+0.002$ to $+0.003$). Two one-factor cells (bilinear mask downsampling, 99th-percentile clipping) stay within $-0.033$ to $+0.009$ across all three methods for both PG and E-PG.}

\rev{Two observations qualify the result. Percentile clipping at 99 is degenerate for Simple Gradient: it collapses PG from about 0.556 to about 0.18 for both AFT$^4$ checkpoints simultaneously (and non-fine-tuned CLIP from 0.261 to 0.073), so the tiny gap in that cell ($-0.004$) is at noise level. Grad-CAM is the one method where the reproduction trails the released checkpoint by a stable margin in every cell ($-0.025$ to $-0.028$ PG, $-0.009$ to $-0.010$ IoU), while Simple Gradient and Grad-ECLIP match to within noise. That stable Grad-CAM offset is small in absolute terms and is consistent with the instability of Grad-CAM on CLIP encoders discussed in Section~\ref{sec:claim1}.}

\section{Network Dissection Ablations}
\label{appendix:nd_ablation}

\rev{Complementing Section~\ref{sec:concept}, we re-ran Network Dissection on Broden (63,305 probing images, 768 neurons) with the detector thresholds, probed layer, and concept vocabulary exposed as configuration. Alignment scoring uses 210,365 images. The sweep covers the arms of the main text plus the author-released AFT$^4$ checkpoint and the COCO-fine-tuned models of Section~\ref{sec:sgr}.}

\begin{table}[h]
\centering
\caption{\rev{Detector counts across the IoU-threshold sweep (activation quantile 0.005, final token). The main-text operating point is IoU 0.04.}}
\label{tab:nd_iou_sweep}
\small
{\color{revcolor}%
\begin{tabular}{lccc}
\toprule
Model & IoU 0.02 & IoU 0.04 & IoU 0.06 \\
\midrule
\multicolumn{4}{l}{\textit{ImageNet fine-tuning only}} \\
CLIP (w/o AFT)     & 572 & 160 & 64 \\
AFT$^1$ (ours)     & 662 & 279 & 122 \\
AFT$^4$ (ours)     & 721 & 390 & 183 \\
AFT$^4$ (original) & 721 & 390 & 197 \\
\midrule
\multicolumn{4}{l}{\textit{COCO fine-tuning}} \\
\;$w_s{=}0$ (control) & 701 & 320 & 130 \\
\;+ EPG $w_s{=}0.2$   & 711 & 326 & 135 \\
\;+ EPG $w_s{=}0.4$   & 714 & 344 & 152 \\
\;+ EPG $w_s{=}0.6$   & 716 & 363 & 178 \\
\bottomrule
\end{tabular}}
\end{table}

\rev{The ordering AFT$^4 > $ AFT$^1 > $ CLIP holds in every ablation cell across the IoU-threshold, activation-quantile, probed-layer, and concept-vocabulary sweeps (the only tie is the saturated cell $q = 0.02$ with IoU 0.02, where all 768 neurons qualify for every model). The absolute detector count varies by a factor of nine for the baseline across the IoU window alone, which brackets both the original paper's counts and ours and identifies the threshold as the dominant source of the count discrepancy discussed in Section~\ref{sec:concept}. The released AFT$^4$ checkpoint matches our retrained one exactly at IoU 0.02 and 0.04 (721 and 390 detectors) and within 14 at 0.06. The COCO-fine-tuned family sits below the ImageNet-trained arms at every threshold, but the $w_s{=}0$ control attributes this to the additional fine-tuning stage rather than to the saliency term: within the family the count rises monotonically with the saliency weight at every threshold (320 to 363 at IoU 0.04, 130 to 178 at 0.06), recovering most of the gap to AFT$^4$ by $w_s{=}0.6$ (363 versus 390, and 178 versus 183). Explicit saliency supervision therefore leaves concept alignment intact, with a small gain if anything.}

\rev{Figure~\ref{fig:alignment_hist} shows the full score distributions behind Table~\ref{tab:concept4}: AFT consistently improves neuron-concept alignment across all three similarity metrics compared to the baseline, with the whole distribution shifting upward monotonically with stronger AFT.}

\begin{figure}[h]
\centering
\includegraphics[width=\textwidth]{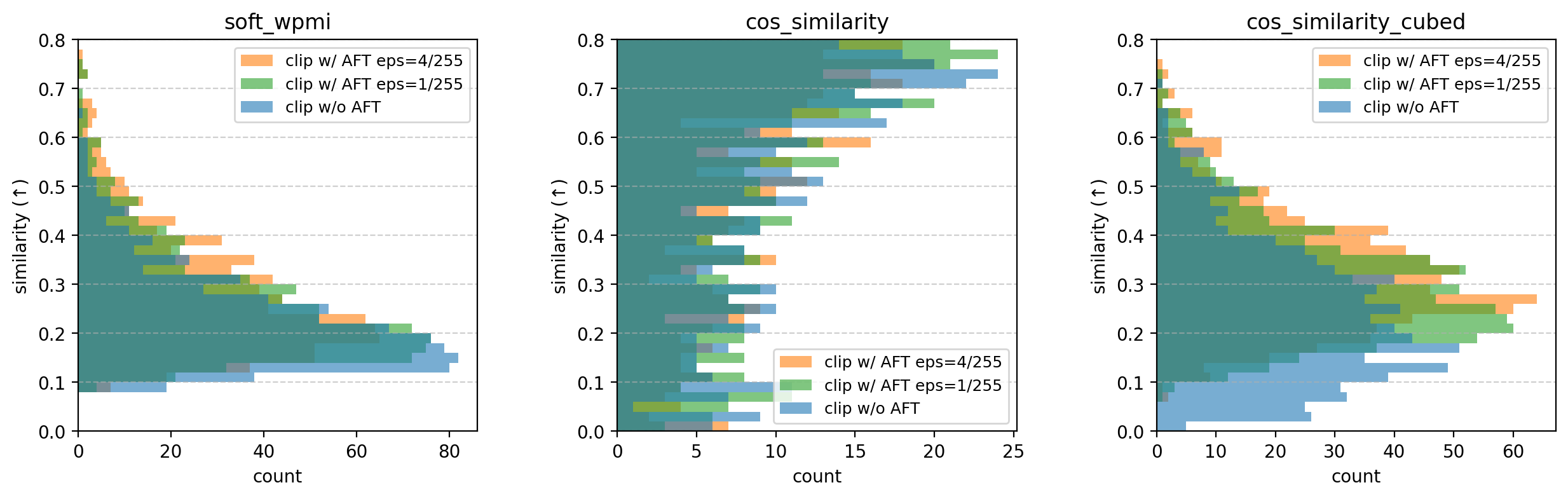}
\caption{Distribution of neuron-concept alignment scores under three
similarity metrics for CLIP with and without AFT.}
\label{fig:alignment_hist}
\end{figure}

\section{Zero-Shot Accuracy: Seed Sensitivity}
\label{appendix:seed_cis}

\rev{\paragraph{Seed sensitivity.} Table~\ref{tab:seed_cis} reports mean accuracy with two-sided 95\% Student-t confidence intervals over five evaluation seeds (42 to 46), each drawing a fresh 1,000-image subset. The attack (untargeted APGD on the cross-entropy loss, then targeted APGD on the DLR loss toward the highest-scoring incorrect class on surviving samples; 50 iterations per stage, no random restarts, no FAB or Square attacks) is deterministic, so the spread is sample-selection variance only. The widest interval in the study is $\pm 2.42$ points (STL-10, AFT$^1$, $\varepsilon = 2/255$). CLIP (w/o AFT) is omitted from the table since it loses all accuracy under both attack budgets on every seed, and its Table~\ref{tab:reproduction_results_complete} entries other than the five-seed Imagenette value keep the fixed seed 42. We also checked the near-identical EuroSAT rows of AFT$^4$ (original) and AFT$^4$ (ours) against the raw per-seed runs. The runs are distinct (separate checkpoints, and per-seed accuracies that differ on every other dataset), yet on EuroSAT the two checkpoints score identically in 14 of the 15 seed-budget cells and differ by 0.1 in one. This is consistent with both models falling back on the same prediction pattern at the chance floor, where per-seed accuracy is set by the subset composition rather than the weights.}

\begin{table}[h]
\centering
\caption{\rev{Zero-shot accuracy, mean $\pm$ 95\% CI over five evaluation seeds.}}
\label{tab:seed_cis}
\resizebox{0.85\textwidth}{!}{
{\color{revcolor}%
\begin{tabular}{llccccc}
\toprule
$L_\infty$ & Model & CIFAR-10 & STL-10 & EuroSAT & FGVC Air. & Imagenette \\
\midrule
\multirow{7}{*}{clean}
& AFT$^4$ (original) & 64.08 $\pm$ 1.58 & 89.96 $\pm$ 0.73 & 10.88 $\pm$ 0.89 & 12.84 $\pm$ 0.98 & 98.22 $\pm$ 0.36 \\
& AFT$^1$ (ours)     & 87.56 $\pm$ 0.90 & 97.38 $\pm$ 0.32 & 10.50 $\pm$ 1.18 & 20.82 $\pm$ 1.78 & 99.26 $\pm$ 0.17 \\
& AFT$^4$ (ours)     & 65.80 $\pm$ 1.45 & 90.88 $\pm$ 0.59 & 10.88 $\pm$ 0.89 & 12.26 $\pm$ 0.90 & 98.60 $\pm$ 0.18 \\
& COCO AFT ($w_s{=}0$) & 55.74 $\pm$ 1.71 & 83.42 $\pm$ 1.08 & 8.26 $\pm$ 0.60 & 13.68 $\pm$ 1.40 & 95.20 $\pm$ 0.28 \\
& + EPG $w_s{=}0.2$  & 52.96 $\pm$ 1.75 & 82.12 $\pm$ 1.09 & 5.36 $\pm$ 0.52 & 14.20 $\pm$ 1.32 & 95.10 $\pm$ 0.33 \\
& + EPG $w_s{=}0.4$  & 51.26 $\pm$ 0.99 & 80.22 $\pm$ 1.01 & 5.20 $\pm$ 0.80 & 13.70 $\pm$ 1.20 & 94.86 $\pm$ 0.41 \\
& + EPG $w_s{=}0.6$  & 50.32 $\pm$ 0.81 & 76.38 $\pm$ 1.61 & 6.22 $\pm$ 0.85 & 12.76 $\pm$ 1.19 & 94.26 $\pm$ 0.57 \\
\midrule
\multirow{7}{*}{2/255}
& AFT$^4$ (original) & 43.66 $\pm$ 1.88 & 77.22 $\pm$ 0.79 & 10.68 $\pm$ 0.73 & 3.84 $\pm$ 0.73 & 94.08 $\pm$ 0.42 \\
& AFT$^1$ (ours)     & 21.50 $\pm$ 1.34 & 58.38 $\pm$ 2.42 & 0.80 $\pm$ 0.15 & 1.40 $\pm$ 0.30 & 83.38 $\pm$ 0.87 \\
& AFT$^4$ (ours)     & 45.52 $\pm$ 1.69 & 78.78 $\pm$ 0.87 & 10.68 $\pm$ 0.73 & 3.58 $\pm$ 0.46 & 94.52 $\pm$ 0.36 \\
& COCO AFT ($w_s{=}0$) & 35.36 $\pm$ 1.40 & 64.68 $\pm$ 0.73 & 2.18 $\pm$ 0.80 & 4.46 $\pm$ 0.84 & 83.66 $\pm$ 1.38 \\
& + EPG $w_s{=}0.2$  & 34.76 $\pm$ 1.15 & 62.26 $\pm$ 1.50 & 1.50 $\pm$ 0.46 & 3.98 $\pm$ 0.82 & 82.84 $\pm$ 1.76 \\
& + EPG $w_s{=}0.4$  & 32.94 $\pm$ 0.65 & 59.20 $\pm$ 1.33 & 1.22 $\pm$ 0.34 & 3.70 $\pm$ 0.67 & 80.44 $\pm$ 1.53 \\
& + EPG $w_s{=}0.6$  & 31.58 $\pm$ 0.91 & 55.76 $\pm$ 1.88 & 1.80 $\pm$ 0.54 & 3.14 $\pm$ 0.72 & 78.56 $\pm$ 1.69 \\
\midrule
\multirow{7}{*}{4/255}
& AFT$^4$ (original) & 22.86 $\pm$ 0.85 & 55.88 $\pm$ 1.50 & 10.66 $\pm$ 0.74 & 1.12 $\pm$ 0.42 & 84.54 $\pm$ 1.19 \\
& AFT$^1$ (ours)     & 0.24 $\pm$ 0.29 & 7.28 $\pm$ 0.85 & 0.00 $\pm$ 0.00 & 0.00 $\pm$ 0.00 & 23.32 $\pm$ 1.39 \\
& AFT$^4$ (ours)     & 24.46 $\pm$ 0.94 & 58.54 $\pm$ 1.51 & 10.64 $\pm$ 0.69 & 1.16 $\pm$ 0.30 & 86.14 $\pm$ 0.81 \\
& COCO AFT ($w_s{=}0$) & 15.32 $\pm$ 1.61 & 38.18 $\pm$ 1.05 & 0.58 $\pm$ 0.39 & 1.22 $\pm$ 0.16 & 59.64 $\pm$ 1.75 \\
& + EPG $w_s{=}0.2$  & 15.26 $\pm$ 1.35 & 36.82 $\pm$ 1.36 & 0.50 $\pm$ 0.35 & 1.18 $\pm$ 0.16 & 58.64 $\pm$ 1.80 \\
& + EPG $w_s{=}0.4$  & 13.86 $\pm$ 1.05 & 34.24 $\pm$ 1.54 & 0.38 $\pm$ 0.33 & 1.08 $\pm$ 0.28 & 55.68 $\pm$ 1.66 \\
& + EPG $w_s{=}0.6$  & 12.56 $\pm$ 0.95 & 30.44 $\pm$ 1.34 & 0.58 $\pm$ 0.32 & 0.90 $\pm$ 0.34 & 52.42 $\pm$ 1.57 \\
\bottomrule
\end{tabular}}
}
\end{table}

\section{Quantitative Transfer Results: Full Tables}
\label{appendix:ood_loc_full}

\rev{Full five-metric versions of Table~\ref{tab:ood_loc} (Tables~\ref{tab:ood_zeroshot_full} and~\ref{tab:ood_coco}). All runs use 1,000 samples, seed 42, topk 0.2, and a zero-tolerance pointing game. SG = Simple Gradient, GC = Grad-CAM, GE = Grad-ECLIP; E-PG = Energy Pointing Game. The CLIP (w/o AFT) row of the COCO table uses a linear probe trained on frozen vanilla-CLIP features, so probe and encoder are matched.}

\rev{The full tables confirm the pattern of Table~\ref{tab:ood_loc}: the EPG dose response is monotone for Simple Gradient and Grad-ECLIP in every column, and the retrained models track the released checkpoint most tightly on Simple Gradient. Grad-CAM shows the largest reproduction gaps, and on Kvasir-SEG its columns fluctuate at near-floor signal levels, so we read them with caution. Table~\ref{tab:flickr30k_retrieval_formatted} reports the full Flickr30k cross-modal retrieval results underlying the discussion in the main text.}

\begin{table}[h]
\centering
\caption{Cross-modal retrieval results (Recall@k, \%) on the Flickr30k dataset under different perturbation budgets. $\mathrm{AFT}^i$ denotes models trained with perturbation $\varepsilon = i / 255$. (ours) denotes retrained models; (original) denotes original checkpoints provided by authors. Bold marks the best value per column within each budget.}
\label{tab:flickr30k_retrieval_formatted}
\resizebox{0.44\textwidth}{!}{
\begin{tabular}{cc|cc|cc}
\toprule
& & \multicolumn{2}{c|}{T$\rightarrow$I (\%)} & \multicolumn{2}{c}{I$\rightarrow$T (\%)} \\
$L_{\infty}$ & \begin{tabular}[c]{@{}c@{}}Vision\\ Encoder\end{tabular} & R@1 & R@5 & R@1 & R@5 \\ 
\midrule \midrule

\multirow{4}{*}{\rotatebox{90}{clean}} 
& CLIP baseline & 24.95 & 45.37 & \textbf{44.41} & \textbf{68.61} \\
& AFT$^1$ (\textit{ours}) & \textbf{27.20} & \textbf{48.52} & 43.02 & 67.09 \\
& AFT$^4$ (\textit{original}) & 19.53 & 37.50 & 26.38 & 47.54 \\
& AFT$^4$ (\textit{ours}) & 19.53 & 37.46 & 26.33 & 47.78 \\

\midrule

\multirow{4}{*}{\rotatebox{90}{2/255}} 
& CLIP baseline & 1.83 & 4.74 & 2.25 & 5.66 \\
& AFT$^1$ (\textit{ours}) & \textbf{18.59} & \textbf{35.95} & \textbf{30.72} & \textbf{52.05} \\
& AFT$^4$ (\textit{original}) & 17.88 & 34.73 & 24.07 & 43.66 \\
& AFT$^4$ (\textit{ours}) & 17.84 & 34.69 & 24.20 & 44.05 \\

\midrule

\multirow{4}{*}{\rotatebox{90}{4/255}} 
& CLIP baseline& 0.18 & 0.57 & 0.22 & 0.72 \\
& AFT$^1$ (\textit{ours}) & 7.80 & 17.15 & 13.60 & 26.43 \\
& AFT$^4$ (\textit{original}) & 13.96 & 27.71 & 18.36 & 34.61 \\
& AFT$^4$ (\textit{ours}) & \textbf{14.07} & \textbf{27.89} & \textbf{18.79} & \textbf{35.51} \\

\bottomrule
\end{tabular}
}
\end{table}

\begin{table}[h]
\centering
\caption{\rev{Zero-shot localisation on ISIC (dermoscopy) and Kvasir-SEG (colonoscopy polyps; small targets). The CLIP and AFT rows are fine-tuned on ImageNet only; the EPG rows are fine-tuned on COCO with COCO instance masks as supervision. Both domains are out-of-distribution for either training set, so the in-domain confound noted for Table~\ref{tab:ood_coco} does not apply here. Bold marks the best value per column; italics mark the best among the ImageNet-only arms.}}
\label{tab:ood_zeroshot_full}
\small
\resizebox{\textwidth}{!}{%
{\color{revcolor}%
\begin{tabular}{llccccc|ccccc|ccccc}
\toprule
& & \multicolumn{5}{c|}{Simple Gradient} & \multicolumn{5}{c|}{Grad-CAM} & \multicolumn{5}{c}{Grad-ECLIP} \\
& Model & PG & E-PG & AP & PA & IoU & PG & E-PG & AP & PA & IoU & PG & E-PG & AP & PA & IoU \\
\midrule
\multirow{9}{*}{\rotatebox[origin=c]{90}{ISIC}}
& \multicolumn{16}{l}{\textit{ImageNet fine-tuning only}} \\
& CLIP (w/o AFT)      & 0.107 & 0.238 & 0.253 & 0.609 & 0.091 & 0.190 & 0.261 & 0.286 & \textbf{\textit{0.622}} & 0.103 & 0.397 & 0.308 & 0.425 & 0.663 & 0.196 \\
& AFT$^4$ (original)  & 0.428 & \textit{0.441} & \textit{0.523} & \textit{0.714} & \textit{0.282} & 0.346 & \textit{0.367} & \textit{0.386} & 0.464 & \textit{0.281} & 0.460 & \textit{0.395} & \textit{0.537} & \textit{0.700} & \textit{0.272} \\
& AFT$^1$ (ours)      & \textit{0.457} & 0.401 & 0.484 & 0.699 & 0.249 & \textit{0.382} & 0.352 & 0.373 & 0.550 & 0.223 & \textit{0.469} & 0.364 & 0.498 & 0.691 & 0.246 \\
& AFT$^4$ (ours)      & 0.424 & 0.436 & 0.511 & 0.709 & 0.274 & 0.290 & 0.344 & 0.372 & 0.558 & 0.246 & 0.426 & 0.377 & 0.516 & 0.691 & 0.255 \\
& \multicolumn{16}{l}{\textit{COCO fine-tuning with EPG supervision}} \\
& + EPG $w_s{=}0.2$   & 0.648 & 0.530 & 0.634 & 0.750 & 0.350 & 0.598 & 0.581 & 0.504 & 0.489 & 0.281 & 0.750 & 0.464 & 0.689 & 0.745 & 0.362 \\
& + EPG $w_s{=}0.4$   & 0.763 & 0.588 & 0.698 & 0.766 & 0.383 & 0.683 & 0.626 & 0.561 & 0.543 & 0.288 & 0.852 & 0.497 & 0.738 & 0.762 & 0.389 \\
& + EPG $w_s{=}0.6$   & \textbf{0.820} & \textbf{0.638} & \textbf{0.747} & \textbf{0.778} & \textbf{0.407} & \textbf{0.753} & \textbf{0.673} & \textbf{0.589} & 0.522 & \textbf{0.294} & \textbf{0.905} & \textbf{0.518} & \textbf{0.766} & \textbf{0.772} & \textbf{0.404} \\
\midrule
\multirow{9}{*}{\rotatebox[origin=c]{90}{Kvasir-SEG}}
& \multicolumn{16}{l}{\textit{ImageNet fine-tuning only}} \\
& CLIP (w/o AFT)      & 0.058 & 0.123 & 0.135 & 0.693 & 0.062 & 0.054 & \textit{0.135} & 0.137 & 0.684 & 0.049 & \textit{0.085} & 0.161 & 0.175 & 0.704 & 0.082 \\
& AFT$^4$ (original)  & 0.051 & 0.172 & 0.200 & 0.725 & 0.121 & 0.036 & 0.082 & 0.174 & 0.343 & \textbf{\textit{0.142}} & 0.020 & 0.153 & 0.181 & 0.700 & 0.082 \\
& AFT$^1$ (ours)      & \textit{0.072} & \textit{0.191} & \textit{0.218} & \textit{0.735} & \textit{0.137} & \textit{0.094} & 0.125 & 0.156 & \textbf{\textit{0.688}} & 0.059 & 0.043 & \textit{0.167} & 0.189 & \textit{0.708} & \textit{0.094} \\
& AFT$^4$ (ours)      & 0.056 & 0.178 & 0.207 & 0.728 & 0.127 & 0.065 & 0.123 & \textit{0.191} & 0.465 & 0.135 & 0.028 & 0.162 & \textit{0.191} & 0.707 & 0.092 \\
& \multicolumn{16}{l}{\textit{COCO fine-tuning with EPG supervision}} \\
& + EPG $w_s{=}0.2$   & 0.113 & 0.194 & 0.228 & 0.737 & 0.143 & 0.048 & 0.106 & 0.180 & 0.486 & 0.124 & 0.037 & 0.182 & 0.221 & 0.726 & 0.123 \\
& + EPG $w_s{=}0.4$   & 0.180 & 0.217 & 0.259 & 0.750 & 0.166 & 0.053 & 0.139 & 0.194 & 0.591 & 0.121 & 0.086 & 0.198 & 0.252 & 0.741 & 0.150 \\
& + EPG $w_s{=}0.6$   & \textbf{0.265} & \textbf{0.249} & \textbf{0.303} & \textbf{0.765} & \textbf{0.197} & \textbf{0.095} & \textbf{0.174} & \textbf{0.211} & 0.593 & 0.130 & \textbf{0.187} & \textbf{0.217} & \textbf{0.304} & \textbf{0.760} & \textbf{0.187} \\
\bottomrule
\end{tabular}}%
}
\end{table}


\begin{table}[h]
\centering
\caption{\rev{COCO 6-class linear-probe localisation (person, car, dog, cat, bird, chair). Grad-ECLIP is not defined for the linear probe. The CLIP and AFT rows are fine-tuned on ImageNet only; the EPG rows are fine-tuned on COCO with COCO instance masks as supervision, so on this COCO-derived task part of their margin may reflect in-domain exposure rather than the EPG term itself. Bold marks the best value per column; italics mark the best among the ImageNet-only arms.}}
\label{tab:ood_coco}
\small
{\color{revcolor}%
\begin{tabular}{lccccc|ccccc}
\toprule
& \multicolumn{5}{c|}{Simple Gradient} & \multicolumn{5}{c}{Grad-CAM} \\
Model & PG & E-PG & AP & PA & IoU & PG & E-PG & AP & PA & IoU \\
\midrule
\multicolumn{11}{l}{\textit{ImageNet fine-tuning only}} \\
CLIP (w/o AFT)      & 0.140 & 0.156 & 0.166 & 0.701 & 0.076 & 0.226 & 0.212 & 0.243 & 0.666 & 0.119 \\
AFT$^4$ (original)  & \textit{0.333} & \textit{0.241} & \textit{0.261} & \textit{0.738} & \textit{0.145} & \textbf{\textit{0.244}} & \textbf{\textit{0.228}} & 0.254 & 0.656 & 0.128 \\
AFT$^1$ (ours)      & 0.324 & 0.227 & 0.254 & 0.735 & 0.139 & 0.230 & 0.216 & 0.250 & 0.654 & 0.120 \\
AFT$^4$ (ours)      & 0.328 & 0.240 & 0.260 & \textit{0.738} & 0.144 & 0.239 & 0.227 & \textbf{\textit{0.257}} & \textbf{\textit{0.680}} & \textbf{\textit{0.129}} \\
\midrule
\multicolumn{11}{l}{\textit{COCO fine-tuning with EPG supervision}} \\
+ EPG $w_s{=}0.2$   & 0.387 & 0.271 & 0.300 & 0.750 & 0.166 & 0.218 & 0.214 & 0.237 & 0.618 & 0.116 \\
+ EPG $w_s{=}0.4$   & 0.430 & 0.299 & 0.329 & 0.757 & 0.181 & 0.226 & 0.214 & 0.235 & 0.632 & 0.113 \\
+ EPG $w_s{=}0.6$   & \textbf{0.467} & \textbf{0.327} & \textbf{0.355} & \textbf{0.763} & \textbf{0.193} & 0.236 & 0.217 & 0.235 & 0.606 & 0.115 \\
\bottomrule
\end{tabular}}
\end{table}

\section{Additional Evaluations of the Saliency-Guided Extension}
\label{appendix:epg_faithfulness}

\rev{Results for the four evaluations of Section~\ref{sec:sgr} that use metrics not aligned with the EPG training objective.}

\rev{\paragraph{ROAR after EPG fine-tuning.} Same protocol as Section~\ref{sec:roar} (zero fill, ResNet-50 retrained from scratch), comparing a model trained with AFT on COCO with saliency loss $w_s{=}0.6$ against one trained without it ($w_s{=}0$). The two arms share an identical recipe, so they differ only in the saliency term. Once pixels are removed, the EPG curve tracks at or below the control in all four settings for Simple Gradient, and for Grad-CAM at the low and middle removal levels, with small crossovers in the Grad-CAM panels at the highest removal levels, where both arms approach the accuracy floor. The margins are modest; we read them as supporting evidence, on a metric not aligned with the training objective, that the pixels highlighted by EPG-supervised saliency maps are more task-relevant than those from the unsupervised baseline (Figure~\ref{fig:roar_epg}).}

\begin{figure}[h]
\centering
\includegraphics[width=\textwidth]{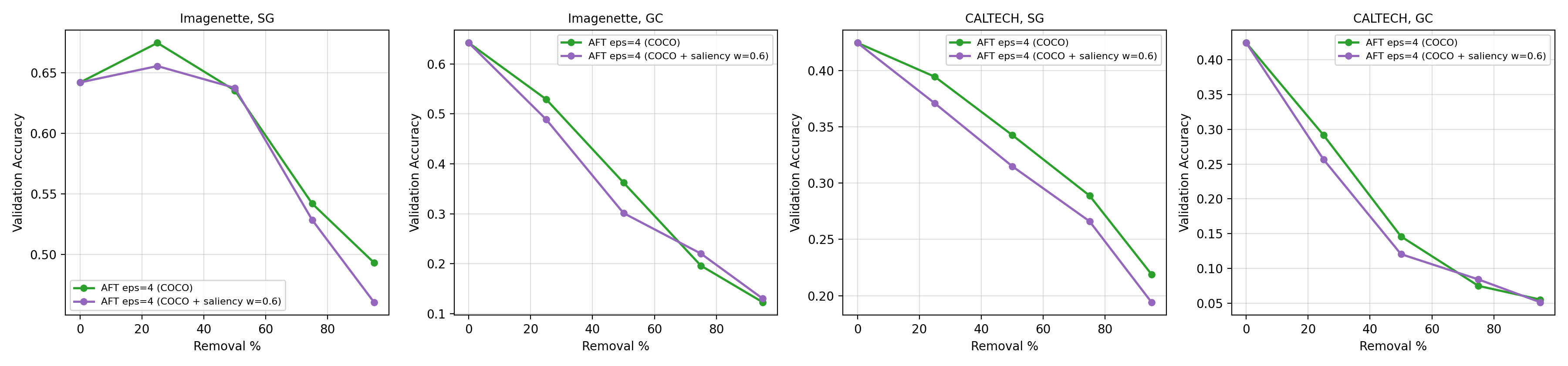}
\caption{\rev{ROAR after EPG fine-tuning. Validation accuracy of a ResNet-50 retrained from scratch on data with the top-$k\%$ most salient pixels removed, for the COCO $w_s{=}0$ control (green) and the $w_s{=}0.6$ model (purple). SG = Simple Gradient, GC = Grad-CAM.}}
\label{fig:roar_epg}
\end{figure}

\rev{\paragraph{Deletion and insertion.} RISE-convention deletion (lower AUC better) and insertion (higher AUC better), 2,000 images per dataset, 50 steps, target = predicted label. Insertion starts from a blurred image and is baseline-independent. The random-ordering control cancels each model's baseline confidence, so the gap between a saliency method and its own random control is the comparable quantity across checkpoints. Under the zeros deletion baseline, Grad-CAM AUC sits above each model's random control for all six checkpoints, consistent with the Grad-CAM instability discussed in Section~\ref{sec:claim1}.}

\begin{table}[!htb]
\centering
\caption{\rev{SmoothGrad separation from the random control. Deletion gap = random $-$ SmoothGrad (positive is good); insertion gap = SmoothGrad $-$ random (positive is good). Zeros deletion baseline. The AFT rows are fine-tuned on ImageNet only; the EPG rows are fine-tuned on COCO with COCO instance masks as supervision. Bold marks the best value per column; italics mark the best among the ImageNet-only arms.}}
\label{tab:delins_gap}
\small
{\color{revcolor}%
\begin{tabular}{lcccc}
\toprule
& \multicolumn{2}{c}{Imagenette} & \multicolumn{2}{c}{ImageNet-val} \\
\cmidrule(lr){2-3} \cmidrule(lr){4-5}
Model & Del gap & Ins gap & Del gap & Ins gap \\
\midrule
\multicolumn{5}{l}{\textit{ImageNet fine-tuning only}} \\
AFT$^4$ (original) & $+0.042$ & $+\textit{0.133}$ & $+0.012$ & $+0.130$ \\
AFT$^1$ (ours)     & $+\textbf{\textit{0.143}}$ & $+0.069$ & $+\textbf{\textit{0.039}}$ & $+0.131$ \\
AFT$^4$ (ours)     & $+0.045$ & $+0.129$ & $+0.014$ & $+\textbf{\textit{0.133}}$ \\
\multicolumn{5}{l}{\textit{COCO fine-tuning with EPG supervision}} \\
+ EPG $w_s{=}0.2$  & $+0.025$ & $+0.152$ & $+0.006$ & $+0.085$ \\
+ EPG $w_s{=}0.4$  & $+0.018$ & $+0.157$ & $+0.005$ & $+0.081$ \\
+ EPG $w_s{=}0.6$  & $+0.017$ & $+\textbf{0.161}$ & $+0.005$ & $+0.072$ \\
\bottomrule
\end{tabular}}
\end{table}

\rev{The gaps (Table~\ref{tab:delins_gap}) are positive in every cell, which validates the metric, but they do not grow with the EPG dose: on the ImageNet validation set the insertion gap shrinks from 0.130 for the AFT baseline to 0.072 at $w_s{=}0.6$, while on Imagenette it moves slightly the other way (0.133 to 0.161).}

\rev{\paragraph{Prompt-specific localisation.} 1,838 COCO val2017 images, each containing two sufficiently large object categories with instance masks. MassRatio = fraction of saliency mass on the prompted object (chance 0.5); $\Delta$ = shift of that mass when the prompt flips (chance 0); PG = pointing game against the prompted mask. Grad-ECLIP discrimination falls monotonically with the dose (relative to AFT$^4$ (ours): mass ratio 0.740 to 0.716, $\Delta$ 0.480 to 0.432) while Simple Gradient stays at chance throughout; the only quantity the dose raises is the Simple Gradient pointing game (0.124 to 0.293).}

\rev{\paragraph{Background-restricted adversarial attack.} Region-restricted PGD (100 steps, 1,000 samples) with the perturbation supported only on the background or only on the foreground mask, testing Proposition~\ref{prop:bg-adv}. $\Phi$ is the feature-consistency damage of Eq.~\ref{eq:bg-adv} (higher means the attack moved the features more), measured on COCO; the ratio column is the within-model background-to-foreground damage ratio. The damage table accompanies the proposition it tests (Table~\ref{tab:bgattack}, Appendix~\ref{sec:theory-epg}). The accuracy view of the same attack (ImageNet-S zero-shot) is mixed and confounded by the EPG models' lower clean accuracy, as discussed in Section~\ref{sec:sgr}.}

\end{document}